\pdfoutput=1
\documentclass[letterpaper]{article}
\usepackage{aaai} 
\usepackage{times}
\usepackage{helvet}
\usepackage{paralist}
\usepackage{courier} 
\usepackage{amsmath}
\usepackage{latexsym} 
\usepackage{amsthm}  
\usepackage{thmtools,thm-restate}
\usepackage{tikz}
\usepackage{macros}
\usetikzlibrary{matrix}
\usepackage{amssymb}
                
\declaretheorem{theorem}
\declaretheorem[sibling=theorem]{proposition}

\declaretheorem[style=definition,sibling=theorem,qed=$\diamond$]{definition}

\allowdisplaybreaks
\nocopyright
\begin{document}
%
\title{Ontology Module Extraction via Datalog Reasoning} 
\author{Ana Armas Romero \and Mark Kaminski \and Bernardo Cuenca Grau \and Ian Horrocks\\
Department of Computer Science, University of Oxford, UK}
\maketitle
\begin{abstract}
Module extraction---the task of computing
a (preferably small) fragment $\M$ of an ontology $\T$ that preserves
entailments over a 
signature $\S$---has found many
applications in recent years. Extracting
modules of minimal size is, however, computationally hard,
and often algorithmically infeasible. Thus, practical techniques
are based on approximations, where $\M$
provably captures the relevant entailments, but is not
guaranteed to be minimal. Existing approximations, however,
ensure that $\M$ preserves all  second-order entailments of $\T$ w.r.t.\ $\S$, 
which is stronger than is required in
many applications, and may lead to 
large modules in practice. 
In this paper we propose
a novel approach in which module extraction is reduced to
a reasoning problem in datalog. Our approach not only generalises
existing approximations in an elegant way, but it
can also be tailored to preserve only specific kinds of entailments,
which allows
us to extract significantly smaller modules. An evaluation on 
widely-used ontologies has shown very encouraging results.
\end{abstract}
 

\section{Introduction}

Module extraction is the task of computing, given an ontology
$\T$ and a signature of interest $\Sigma$, a (preferably small) subset $\M$ of
$\T$ (a module) that preserves all relevant entailments
in $\T$ over the set of symbols $\Sigma$. 
Such an $\M$ is indistinguishable from $\T$ w.r.t.\ $\Sigma$, 
and $\T$ can be safely replaced with $\M$
in applications of $\T$ that use only the symbols in $\Sigma$. 

Module extraction has received a great deal of attention in recent years
\cite{DBLP:series/lncs/5445,GrauHKS08JAIR,DBLP:conf/www/SeidenbergR06,DBLP:journals/ai/KontchakovWZ10,DBLP:conf/dlog/GatensKW14,DBLP:conf/ijcai/VescovoPSS11,DBLP:conf/lpar/NortjeBM13}, 
and modules have found a wide range of applications, including
ontology reuse \cite{GrauHKS08JAIR,DBLP:conf/esws/Jimenez-RuizGSSL08},  matching \cite{DBLP:conf/semweb/Jimenez-RuizG11},
debugging \cite{DBLP:conf/aswc/SuntisrivarapornQJH08,LudwigORE} and
classification \cite{DBLP:conf/semweb/RomeroGH12,DBLP:conf/ore/TsarkovP12,DBLP:journals/jar/GrauHKS10}.
   
The preservation of relevant entailments is formalised via
\emph{inseparability relations}.
The strongest notion is \emph{model} inseparability, which requires
that it must be possible
to turn any
model of $\M$ into a model of $\T$
by (re-)interpreting only the symbols outside $\Sigma$; such an 
$\M$ preserves all second-order
entailments of $\T$ w.r.t.\ $\Sigma$ \cite{DBLP:journals/ai/KonevL0W13}.
A weaker and more flexible notion is \emph{deductive} inseparability,
which requires only that
$\T$ and $\M$ 
entail the same $\Sigma$-formulas \emph{in a given 
query language}.
Unfortunately, the decision problems
associated with module extraction are invariably of high complexity, 
and often undecidable.
For model inseparability, checking whether $\M$ is a $\Sigma$-module
in $\T$ 
is undecidable  even if $\T$ is restricted to be
in the description logic (DL) $\mathcal{EL}$,
for which standard reasoning is tractable.
For deductive inseparability, the problem is typically decidable for lightweight DLs and ``reasonable'' query languages, albeit of high worst-case complexity; e.g., the problem is already
\textsc{ExpTime}-hard for $\mathcal{EL}$ if we consider concept inclusions as
the query language \cite{DBLP:journals/jsc/LutzW10}.
Practical algorithms that ensure minimality of
the extracted modules are known only for 
acyclic $\mathcal{ELI}$ \cite{DBLP:journals/ai/KonevL0W13}
and DL-Lite \cite{DBLP:journals/ai/KontchakovWZ10}.

Practical module extraction techniques are typically based on
sound approximations: they ensure that the extracted fragment
$\M$ is a module (i.e., inseparable from $\T$ w.r.t.\ $\Sigma$),
but they give no minimality guarantee. The most popular
such techniques are based on a family of polynomially checkable
conditions called syntactic 
locality \cite{DBLP:conf/www/GrauHKS07,GrauHKS08JAIR,WhichKindOfModule}; in
particular, $\bot$-locality  and 
$\top\!\bot\!^*$-locality.  
Each locality-based module $\M$ enjoys a number of desirable properties
for applications:
\begin{inparaenum}[\it (i)]
 \item it is model inseparable from $\T$;
 \item it is \emph{depleting}, in the sense that $\T \setminus \M$
 is inseparable from the empty ontology w.r.t.\ $\Sigma$; 
 \item it contains all justifications  (a.k.a.\ explanations) in $\T$
 of every $\Sigma$-formula entailed by $\T$; and
\item last but not least, it can be computed
efficiently, even for very expressive
ontology languages.
\end{inparaenum}

Locality-based techniques are easy to implement, and
surprisingly effective in practice. Their main drawback is that the
extracted modules can be rather large, which limits their 
usefulness in some applications \cite{VescovoKPS0T13}.
One way to address this issue is to develop techniques that more closely approximate minimal modules while still preserving properties \emph{(i)}--\emph{(iii)}.
Efforts in this direction have confirmed that locality-based modules can
be far from optimal in practice \cite{DBLP:conf/dlog/GatensKW14};
however, these techniques
apply only to rather restricted ontology languages and utilise algorithms with high worst-case
\mbox{complexity}.

Another approach to computing smaller modules is to weaken 
properties \emph{(i)}--\emph{(iii)}, which are stronger
than is required in many applications.
In particular, model inseparability (property \emph{(i)}) is
a very strong condition, and
deductive inseparability would usually suffice,
with the query language determining which kinds of 
consequence are preserved; in modular classification, for example,
only atomic concept inclusions need to be preserved.
However, 
all practical module extraction techniques that are applicable
to expressive ontology languages yield modules satisfying
all three properties, and  hence potentially 
much larger than they need to be.

 
In this paper,   
we propose
a technique that reduces module 
extraction to a reasoning problem in datalog.
The connection between module extraction and datalog 
was first observed in \cite{DBLP:conf/esws/Suntisrivaraporn08}, where it was shown that 
locality $\bot$-module extraction for $\mathcal{EL}$ ontologies 
could be reduced to propositional datalog
reasoning.
Our approach takes this connection much farther by
generalising both locality-based and
reachability-based \cite{DBLP:conf/lpar/NortjeBM13} modules
for expressive ontology languages in an elegant way.
A key 
distinguishing feature of our technique is that it
can extract deductively inseparable modules, with the
query language tailored to the requirements of the
application at hand, 
which allows
us to relax Property \emph{(i)} and extract
significantly smaller modules. In all cases our modules preserve
the nice features of locality: they are
widely applicable (even beyond DLs), they
can be efficiently computed,  
they
are depleting  (Property \emph{(ii)}) and they
 preserve all justifications of relevant entailments
(Property \emph{(iii)}). 

We have implemented our approach
using the RDFox datalog engine \cite{DBLP:conf/aaai/MotikNPHO14}. 
Our proof of concept
evaluation shows that
module size consistently decreases as we consider weaker inseparability
relations, which could significantly improve the usefulness of modules in applications.

All our proofs are deferred to the appendix.


\section{Preliminaries}
\label{sec:prelim}

\subsubsection{Ontologies and Queries} We use standard first-order
logic  and assume familiarity with
description logics, ontology languages and theorem proving. 
A signature $\Sigma$ is a set of predicates and
 $\sig{F}$ denotes the signature of a set of formulas $F$.
It is assumed that the nullary falsehood predicate $\bot$ belongs to every $\S$. 
To capture a wide range of KR languages, we formalise
ontology axioms as  \emph{rules}: function-free sentences 
of the form
$\forall \x.[\fml(\x)
\rightarrow\exists\y.[\bigvee_{i=1}^n\fmm_i(\x,\y)]]$,
where $\fml$, $\fmm_i$ are conjunctions of distinct atoms. 
Formula $\fml$ is the rule \emph{body}  and $\exists
\y.[\bigvee_{i=1}^n\fmm_i(\x,\y)]$ is the \emph{head}. 
Universal quantification is omitted for brevity.
Rules are required to be
safe (all variables in the head occur in the body) and
we assume
w.l.o.g.\ that $\top$ (resp.\ $\bot$) does not occur in rule heads 
(resp.\ in rule bodies). A TBox
$\T$ is a finite set of rules; 
TBoxes mentioning equality $(\approx)$ are extended
with its standard axiomatisation.
A fact $\fct$ is a function-free ground atom. 
An ABox $\abox$ is a finite set of facts.
A \emph{positive existential query (PEQ)} is a formula $q(\x) = \exists \y.
\fml(\x,\y)$, where $\fml$ is built from function-free atoms using only $\land$ and~$\lor$.

\subsubsection{Datalog}
A rule is \emph{datalog} if its head has at most one atom
and all variables are universally quantified.
A \emph{datalog program} $\prog$ is a set of datalog rules.
Given $\prog$ and an ABox $\abox$, their
\emph{materialisation} is the
set of facts entailed by $\prog \cup \abox$, which
can be computed by means of forward-chaining. A fact $\fct$ is a consequence of a datalog rule 
$r=\bigwedge_{i=1}^n\fct'_i\to\fcu$ and facts
$\fct_1,\dots,\fct_n$ if $\fct = \fcu\s$  with $\s$
a most-general unifier (MGU) of $\fct_i,\fct'_i$ for each $1 \leq i \leq n$.
A (forward-chaining) \emph{proof}  
of $\fct$ in $\prog \cup \abox$  is a pair $\rho =
(T, \lambda)$ where $T$ is a tree, $\lambda$ is a mapping from 
nodes in $T$ to facts,
and from edges in $T$ to rules in $\prog$,
such that for each node $v$
the following holds:
\begin{inparaenum}
\item $\lambda(v) = \fct$ if $v$ is the root of $T$;
\item $\lambda(v) \in
\abox$ if $v$ is a leaf;
and  
\item if $v$ has children
$w_1,\dots,w_n$ then each edge from $v$
to $w_i$ is labelled by the same rule $r$ and
$\lambda(v)$ is a consequence of $r$
and $\lambda(w_1), \dots, \lambda(w_n)$.
\end{inparaenum}
Forward-chaining is sound and complete:
a fact $\fct$ is in the materialisation 
of $\prog\cup\abox$ iff it has a proof in
$\prog \cup \abox$.
Finally, the \emph{support} of  $\fct$ is
the set of rules occurring in some proof of
$\fct$ in $\prog \cup \abox$.

\subsubsection{Inseparability Relations \& Modules}
We next recapitulate the most common 
inseparability relations studied in the literature.
We say that 
TBoxes $\T$ and 
$\T'$ are 
\begin{itemize}
  \item \emph{$\S$-model inseparable} ($\T\eqm_{\S}\T'$), 
if for every model $\calI$ of $\T$ (resp.\ of $\T'$) 
there exists a model $\calJ$ of $\T'$ (resp.\ of $\T$) with the same domain
s.t.\ $\A^{\calI}=\A^{\calJ}$ for each $\A\in\S$.
 \item \emph{$\S$-query inseparable} ($\T\eqq_\S\T'$) if for every
    Boolean PEQ $q$ and $\S$-ABox $\abox$ we have
    $\T\cup\abox\models q$ iff $\T'\cup\abox\models q$.
 \item \emph{$\S$-fact inseparable} ($\T\eqf_{\S}\T'$) if for every
    fact $\fct$ and ABox $\abox$ over $\S$ we have $\T\cup\abox
    \models \fct$ iff $\T'\cup \abox\models \fct$.\qedhere
 \item \emph{$\S$-implication inseparable} ($\T\eqi_{\S}\T'$) if for
 each $\varphi$ of the form $\A(\x)\rightarrow\B(\x)$ with 
 $\A,\B\in\S$,  \hbox{$\T\models
    \varphi$} iff $\T'\models
  \varphi$.
\end{itemize}

These relations are naturally ordered
from strongest to weakest:
${\eqm_{\S}}\subsetneq{\eqq_{\S}}\subsetneq{\eqf_{\S}}\subsetneq{\eqi_{\S}}$
for each
non-trivial $\Sigma$.

Given an inseparability relation $\equiv$ for $\S$, 
a subset $\M\subseteq \T$ is a \emph{$\equiv$-module of\/
$\T$} if $\T\equiv \M$. Furthermore, $\M$ is \emph{minimal}
if no $\M' \subsetneq \M$ is a $\equiv$-module of\/
$\T$.


\section{Module Extraction via Datalog Reasoning} \label{sec:module-settings}

In this section, we present our approach to module extraction
by reduction into a reasoning problem in datalog. 
Our approach builds on recent techniques
that exploit datalog engines for 
ontology reasoning 
\cite{DBLP:conf/ijcai/KontchakovLTWZ11,stefanoni2013introducing,DBLP:conf/aaai/ZhouNGH14}.
In what follows, we fix an arbitrary
TBox $\T$ and signature $\S\subseteq\sig{\T}$. Unless otherwise stated,
our definitions and theorems are parameterised
by such $\T$ and $\S$. We assume
w.l.o.g.\ that rules in $\T$ do not
share existentially quantified variables.
For simplicity, we also 
assume that $\T$ contains no constants 
(all our results can be
seamlessly extended). 

\subsection{Overview and Main Intuitions}\label{sec:overview}

Our overall strategy to
extract a module $\M$  of $\T$ 
for an inseparability relation $\equiv^z_{\S}$,
with $z \in \set{\mathsf{m},\mathsf{q},\mathsf{f},\mathsf{i}}$, 
can be summarised by the following steps:
\begin{enumerate}
  \item Pick a substitution $\theta$ mapping all existentially 
  quantified variables in $\T$ to 
  constants, and transform $\T$ into a datalog program $\prog$ by 
  \begin{inparaenum}[\it(i)] 
  \item Skolemising 
all rules in $\T$ using $\theta$ and
\item turning disjunctions into conjunctions while splitting them into different rules,
thus replacing each function-free disjunctive rule of the form 
$\fml(\x) \rightarrow \bigvee_{i=1}^n\fmm_i(\x)$
with datalog rules 
$\fml(\x) \rightarrow \fmm_1(\x), \dots, \fml(\x) \rightarrow \fmm_n(\x)$.
  \end{inparaenum}

 \item Pick a $\Sigma$-ABox 
$\abox_0$ and materialise
$\prog \cup \abox_0$.
 \item Pick a set $\arel$ of ``relevant facts'' in the materialisation
 and compute the supporting rules in $\prog$ for
 each such fact.
\item 
The module $\M$ consists of all rules in $\T$ that
yield some supporting rule in $\prog$. In this way,
$\M$ is fully determined
by the substitution $\theta$ and the ABoxes $\abox_0$ and $\arel$.
\end{enumerate}
The main intuition behind our module extraction approach is that we
can pick $\theta$, $\abox_0$ and $\arel$ (and hence $\M$) such 
that each proof $\rho$ of a $\S$-consequence $\varphi$ of $\T$ to
be preserved can be embedded in a forward chaining proof $\rho'$
in $\prog \cup \abox_0$
of a relevant fact in $\arel$.
Such an embedding satisfies the key property that, for each rule $r$
involved in $\rho$, at least one corresponding datalog rule in $\prog$
is involved in $\rho'$. In this way we ensure that $\M$ contains the
necessary rules to entail $\varphi$.
This approach, however, does not
ensure minimality of $\M$: since 
$\prog$ is a strengthening of $\T$ there may be
proofs of a relevant fact in $\prog \cup \abox_0$ that do not
correspond to a $\S$-consequence of $\T$, which may lead to
unnecessary rules in~$\M$.

\begin{figure}
\[
{\small
\begin{array}{rl|r@{\:\sby\:}l}
(\ax_1) 	& \A(x) \rightarrow \exists y_1. [\roleR(x,y_1) \wedge \B(y_1)] 		& \A&\exists \roleR. \B \\
(\ax_2) 	& \A(x) \rightarrow \exists y_2. [\roleR(x,y_2) \wedge \C(y_2)] 		& \A&\exists \roleR. \C \\
(\ax_3) 	& \B(x)\wedge\C(x) \rightarrow \D(x) 											& \B\sqcap\C&\D \\
(\ax_4) 	& \D(x) \rightarrow \exists y_3. [\roleS(x,y_3) \wedge \E(y_3)] 		& \D&\exists \roleS. \E \\
(\ax_5) 	& \D(x) \wedge \roleS(x,y) \rightarrow \F(y) 									& \D&\forall \roleS. \F \\
(\ax_6) 	& \roleS(x,y) \wedge \E(y) \wedge \F(y) \rightarrow \G(x) 				& \exists \roleS. (\E\sqcap\F)&\G \\
(\ax_7) 	& \G(x) \wedge \H(x) \rightarrow \bot		& \G \sqcap \H& \bot 
\end{array}}\]
\caption{Example TBox $\Tex$ with DL translation}
\label{fig:exampleTBox}
\end{figure}

\begin{figure*}[t]
    \begin{small}
      \begin{tikzpicture}[parent anchor=south,sibling distance=15mm,level distance=10mm,>=stealth]
        \tikzstyle{level 2}=[sibling distance=18mm]
        \tikzstyle{level 3}=[sibling distance=15mm]
        \hskip20mm
        \node (top) {}
        child[grow=180,level distance=40mm] {node (leftTree) {$\G(a)$}
          [grow=-90,level distance=10mm]
          child {node (Safa1) {$\roleS(a,f(a))$}
            child {node {$\D(a)$}
              edge from parent node[right=-0.3mm] {$\ax'_4$}}}
          child[grow=-115,level distance=11mm]
          {node (Efa) {$\hskip5pt\E(f(a))$}
            child[grow=-90,level distance=10mm] {node {$\D(a)$}
              edge from parent node[right=-0.3mm] {$\ax''_4$}}
            edge from parent node[right=0mm] {$\ax_6$}}
          child {node (Ffa) {$\F(f(a))$}
            child {node (Safa2) {$\roleS(a,f(a))$}
              child {node {$\D(a)$}
                edge from parent node[right=-0.3mm] {$\ax'_4$}}}
            child {node {$\D(a)$}}}
          edge from parent[draw=none]}
        child[grow=0,level distance=40mm] {node (rightTree) {$\G(a)$}
          [grow=-90,level distance=10mm]
          child {node (Sac1) {$\roleS(a,c)$}
            child {node {$\D(a)$}}}
          child[grow=-115,level distance=11mm] {node (Ec) {$\hskip0pt\E(c)$}
            child[grow=-90,level distance=10mm] {node {$\D(a)$}}}
          child {node (Fc) {$\F(c)$}
            child {node (Sac2) {$\roleS(a,c)$}
              child {node {$\D(a)$}}}
            child {node {$\D(a)$}}}
          edge from parent[draw=none]};
        \node[below of=Ffa,node distance=5.7mm] {$\ax_5$};
        \node[below of=leftTree,node distance=35mm] {(a)};
        \node[below of=rightTree,node distance=35mm] {(b)};
        \draw[->,dotted,shorten <=-5pt,shorten >=-15pt] (Safa1.25) .. controls +(3,1) and +(-3,1) .. (Sac1.140) node[midway,below] {$\theta$};
        \draw[->,dotted,shorten <=-5pt,shorten >=-12pt] (Ffa.-35) .. controls +(3,-1) and +(-3,-1) .. (Fc.-140) node[near start,above] {$\theta$};
        \draw[->,dotted,shorten <=-5pt,shorten >=-12pt] (Safa2.-25) .. controls +(3,-1) and +(-3,-1) .. (Sac2.-125) node[midway,above] {$\theta$};
        \draw[->,dotted,shorten <=-5pt,shorten >=-12pt] (Efa.-35) .. controls +(3,-1) and +(-3,-1) .. (Ec.-140) node[midway,below] {$\theta$};
        \end{tikzpicture}
      \end{small}
  \caption{Proofs of $\G(a)$ from $\D(a)$ in (a) $\Tex$ and (b) the corresponding datalog program}
  \label{fig:proof-trees}
\end{figure*}
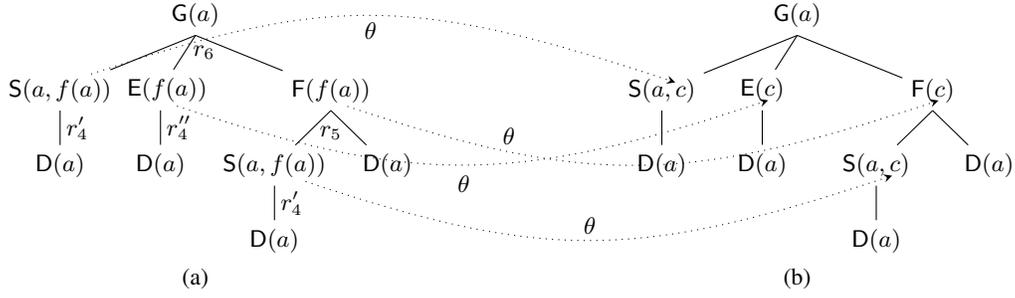


To illustrate how our strategy might work in practice, 
suppose that $\T$ is $\Tex$
in Fig.\ \ref{fig:exampleTBox}, 
$\S =\set{\B,\C,\D,\G}$, and that
we want a module $\M$  that
is $\S$-implication inseparable from $\Tex$. This is a
 simple case 
since $\varphi=\D(x)\to\G(x)$ is
the only non-trivial $\S$-implication
entailed
by $\Tex$; thus, for $\M$ to be a module
we only require that $\M \models \varphi$.

Proving $\Tex \models \varphi$ 
amounts to proving $\Tex \cup \{\D(a)\} \models \G(a)$ (with $a$ a fresh constant).
Figure~\ref{fig:proof-trees}(a) depicts a hyper-resolution tree 
$\rho$ showing how $\G(a)$ can be derived
from the clauses corresponding to $\ax_4$--$\ax_6$ and $\D(a)$, with
rule $\ax_4$ transformed into clauses
\begin{align*}
\ax_4' = \D(x) \to \roleS(x,f(x_3)) & &  \ax_4'' = \D(x) \to \E(f(x_3))
\end{align*}
Hence $\M = \{\ax_4$--$\ax_6\}$ is a $\S$-implication 
inseparable module of $\Tex$, and as 
$\G(a)$ cannot be derived
from any subset of $\{\ax_4$--$\ax_6\}$, $\M$ is also minimal.

In our approach, we pick
$\abox_0$ to contain the initial fact $\D(a)$, $\arel$ to contain
the fact to be proved $\G(a)$, and we make $\theta$
map variable $y_3$ in $\ax_4$
to a fresh constant $c$, in which case rule
$\ax_4$ corresponds to the following datalog rules in $\prog$:
\begin{align*}
\D(x) \to \roleS(x,c) && \D(x) \to \E(c)
\end{align*}
Figure~\ref{fig:proof-trees}(b) depicts a forward chaining
proof $\rho'$ of $\G(a)$ in 
$\prog \cup \{\D(a)\}$.
As shown in the figure, $\rho$ can be embedded in $\rho'$ via $\theta$ by
mapping functional terms over $f$ 
to the fresh constant $c$. In this way, the rules
involved in $\rho$ are mapped to the datalog rules involved in $\rho'$
via $\theta$. Consequently, we will extract the (minimal) module $\M = \{\ax_4$--$\ax_6\}$. 

\subsection{The Notion of Module Setting}

The substitution $\theta$ and the ABoxes $\abox_0$
and $\arel$, which determine the extracted module, can be chosen
in different ways to ensure the preservation of different kinds of $\Sigma$-consequences.
The following notion of
a \bla{} captures in a declarative way the main elements of
our approach.

\begin{definition}
\label{def:bla}
A \emph{\bla} for $\T$ and $\S$ is a tuple 
$\upchi = \langle \theta, \abox_0, \arel \rangle$ 
with
$\theta$ a substitution from existentially quantified
variables in $\T$ to constants,
$\abox_0$ a $\Sigma$-ABox,
$\arel$ a \hbox{$\sig{\T}$}-ABox, and s.t.\
no constant in $\upchi$ occurs in $\T$.

The \emph{program} of $\upchi$ is the smallest datalog program
$\prog^{\chi}$ containing, for each
$\ax=\fml(\x)\to\exists\y.[\bigvee_{i=1}^n\fmm_i(\x,\y)]$ in $\T$, the
rule $\fml \rightarrow \bot$ if $n = 0$ and all rules $\fml \to
\fct\theta$ for each $1 \leq i \leq n$ and each atom $\fct$ in
$\psi_i$.  The \emph{support} of $\upchi$ is the set of rules $r \in
\prog^{\chi}$ that support a fact from $\arel$ in $\prog^{\chi} \cup
\abox_0$.  The \emph{module} $\M^{\chi}$ of $\upchi$ is the set of
rules in $\T$ that have a corresponding datalog rule in the support of
$\upchi$.
\end{definition}

\subsection{Modules for each Inseparability Relation}\label{sec:all-modules}

We next consider each
inseparability relation  $\equiv^z_{\S}$, where
$z \in \set{\mathsf{m},\mathsf{q},\mathsf{f},\mathsf{i}}$, and
formulate a specific setting $\upchi_z$
which provably yields
a $\equiv^z_{\S}$-module of $\T$. 

\subsubsection{Implication Inseparability}

 
The example in Section~\ref{sec:overview}
suggests a natural setting $\upchiimp=\langle \theta, \abox_0, \arel \rangle$ 
that guarantees implication inseparability. 
As in our example, we pick 
$\theta$ to be as ``general'' as possible by
Skolemising each existentially quantified variable to 
a fresh constant.
For $\A$ and $\B$ predicates of the same arity $n$,
proving that $\T$ entails a $\S$-implication 
$\varphi=\A(x_1, \ldots, x_n)\to\B(x_1, \ldots, x_n)$, 
amounts
to showing that $\T \cup \{\A(a_1, \ldots, a_n)\} \models \B(a_1, \ldots, a_n)$
for fresh constants $a_1, \ldots, a_n$. Thus, following the ideas of our example,
we 
initialise $\abox_0$ with 
a fact  $\A(c_{\A}^1, \ldots, c_{\A}^n)$ for each
$n$-ary predicate $\A \in \S$, and $\arel$ with
a fact 
$\B(c_{\A}^1, \ldots, c_{\A}^n)$
for each pair of $n$-ary predicates $\{\B,\A\} \subseteq \S$ with 
$\B \neq \A$.

\begin{definition} \label{def:chiimp} %
For each existentially quantified variable $y_j$ in $\T$, let $c_{y_j}$
be a fresh constant. Furthermore, 
for each $\A\in\S$ of arity $n$, let $c_{\A}^1, \ldots, c_{\A}^n$ be 
also fresh constants.  The setting
  $\upchiimp=\langle \theta^{\mathsf{i}}, \abox_0^{\mathsf{i}}, \arel^{\mathsf{i}} \rangle$  is defined as follows:
  \begin{itemize}
  \item $\theta^{\mathsf{i}}=\mset{y_j\mapsto c_{y_j}}{y_j \text{ existentially quantified in }
      \T}$,
  \item $\abox_0^{\mathsf{i}}=\{\A(c_{\A}^1, \ldots, c_{\A}^n) \mid \A \text{~$n$-ary predicate in } \S\}$, and
  \item $\arel^{\mathsf{i}}\,{=}\,\{\B(c_{\A}^1, \ldots, c_{\A}^n)\,{\mid}\,
    \A \neq \B\, \text{~$n$-ary predicates in } \S \}$.\qedhere 
  \end{itemize}  
\end{definition}
%
%
The setting
$\upchiimp$ is reminiscent of the datalog encodings 
typically used  to check whether a
concept  $A$ is subsumed by concept $B$
w.r.t.\  a  ``lightweight'' 
ontology $\T$ \cite{DBLP:conf/semweb/KrotzschRH08,stefanoni2013introducing}.
There, variables in rules are Skolemised as
fresh constants to produce a datalog program $\prog$
and it is then checked whether $\prog \cup \{A(a)\} \models B(a)$.
 
\begin{restatable}{theorem}{datalogAtomicSubsumptionModules}
\label{thm:datalogAtomicSubsumptionModules}
$\M^{\chiimp}\eqi_{\S} \T$.  
\end{restatable}

\subsubsection{Fact Inseparability}

The setting $\upchiimp$ in
Def.~\ref{def:chiimp} cannot be used to ensure fact inseparability.
Consider again $\Tex$ and $\S = \set{\B,\C, \D,\G}$, for which
$\M^{\chiimp} = \{\ax_4, \ax_5, \ax_6\}$. For $\abox =
\set{\B(a),\C(a)}$ we have $\Tex \cup \abox \models \G(a)$ but
$\M^{\chiimp} \cup \abox \not\models \G(a)$, and hence $\M^{\chiimp}$
is not fact inseparable from $\Tex$.

More generally, 
$\M^{\chiimp}$ is only guaranteed to preserve
$\S$-fact entailments \mbox{$\T\cup\abox\models\fct$} where 
$\abox$ is a singleton. However, for a module to be fact inseparable
from $\T$ it must preserve all $\S$-facts when coupled with
\emph{any} $\S$-ABox.   
 We achieve this by choosing 
 $\abox_0$ to be the \emph{critical ABox}
for~$\S$, which consists of all facts that can be
constructed using $\S$ and a single fresh constant
\cite{DBLP:conf/pods/Marnette09}. Every $\S$-ABox
can be homomorphically mapped into the critical
$\S$-ABox. 
In this way, we can show
that all proofs of a $\S$-fact in $\T \cup \abox$ can 
be embedded in a proof of a relevant fact in $\prog^{\chi} \cup \abox_0$.

\begin{definition} \label{def:chidd} %
  Let constants $c_{y_i}$ be as in Def.\ \ref{def:chiimp},
  and let $*$ be a fresh constant. The setting
  $\upchidd=\langle \theta^{\mathsf{f}}, \abox_0^{\mathsf{f}}, \arel^{\mathsf{f}} \rangle$ is defined as follows:
  \begin{inparaenum}[\it (i)]
  \item $\theta^{\mathsf{f}} = \theta^{\mathsf{i}}$,
  \item $\abox_0^{\mathsf{f}} = \mset{\A(\ast,\ldots,\ast)}{\A\in\S}$, and
  \item $\arel^{\mathsf{f}} = \abox_0^{\mathsf{f}}$\qedhere
  \end{inparaenum}
\end{definition}
The datalog programs for $\upchiimp$ and  $\upchidd$ coincide
and hence the only difference between the two settings is in the definition of
their corresponding ABoxes.  In our example, both
$\abox_0^{\mathsf{f}}$ and $\arel^{\mathsf{f}}$ contain facts
$\B(\ast)$, $\C(\ast)$, $\D(\ast)$, and $\G(\ast)$. Clearly,
$\prog^{\chidd} \cup \abox_0 \models \G(\ast)$ and the 
proof additionally involves rule $\ax_3$. Thus $\M^{\chidd} = \set{\ax_3,\ax_4,\ax_5,\ax_6}$.

\begin{restatable}{theorem}{disjDatalogModules}
\label{thm:disjDatalogModules}
$\M^{\chidd}\eqf_{\S}\T$.
\end{restatable}
%

\subsubsection{Query Inseparability}
Positive existential queries constitute a much richer query 
language than facts as they
allow for existentially quantified variables. 
Thus, the
query inseparability requirement  invariably 
leads to larger modules.

For instance, let $\T=\Tex$ and
$\S=\set{\A,\B}$. Given the $\S$-ABox
$\abox = \{\A(a)\}$ and $\S$-query $q = \exists y.\B(y)$ we have that
$\Tex \cup \abox \models q$ (due to rule $\ax_1$). 
The module
$\M^{\chidd}$ is, however, empty. Indeed, the materialisation 
of $\prog^{\chidd} \cup \{\A(\ast)\}$ consists of the additional facts
$\roleR(\ast,c_{y_1})$ and $\B(c_{y_1})$ and hence it does not
contain any relevant fact mentioning only $\ast$.
Thus, 
$\M^{\chidd} \cup \abox \not\models q$ and $\M^{\chidd}$ is not
query inseparable from $\Tex$.

Our example suggests that, although the critical ABox
is constrained enough to embed every $\S$-ABox,  we may need
to consider additional relevant facts to capture all
proofs of a $\S$-query. 
In particular, rule $\ax_1$ implies that $\B$ contains an 
instance
whenever $\A$ does: a dependency that is then checked by $q$.
This can be captured by considering fact $\B(c_{y_1})$
as relevant, in which case rule $\ax_1$ would be in the module.

More generally, we consider a module setting $\upchi$ 
that differs from $\upchidd$
only in that all $\S$-facts (and not just those over 
$\ast$) are considered
as relevant. 
\begin{definition} \label{def:chicq} %
  Let constants $c_{y_i}$ and $\ast$ be as in Def.~\ref{def:chidd}. 
  The setting $\upchicq=\langle \theta^{\mathsf{q}}, \abox_0^{\mathsf{q}}, \arel^{\mathsf{q}} \rangle$ is 
  as follows:
  \begin{inparaenum}[\it (i)]
  \item $\theta^{\mathsf{q}}= \theta^{\mathsf{f}}$,
  \item $\abox_0^{\mathsf{q}} = \abox_0^{\mathsf{f}}$, and 
  \item $\arel^{\mathsf{q}}$ consists of all $\S$-facts $\A(a_1, \ldots, a_n)$
  with each $a_j$ either a constant $c_{y_i}$ or $\ast$.\qedhere
  \end{inparaenum}
\end{definition}
Correctness is established by the following theorem:

\begin{restatable}{theorem}{datalogCQModules} \label{thm:datalogCQModules}
  $\M^{\chicq}\eqq_{\S}\T$.
\end{restatable}

\subsubsection{Model Inseparability}
\label{sec:semantic-modules}

The modules generated by $\chicq$ may not be model inseparable from
$\T$. To see this, let $\T=\Tex$ and $\S = \{\A,\D,\roleR\}$, in which case
$\M^{\chicq}=\set{\ax_1,\ax_2}$. The interpretation 
$\calI$ where 
$\Delta^{\calI} = \{a,b\}$, $\A^{\calI} = \{a\}$,
$\B^{\calI} = \C^{\calI} = \{b\}$, $\D^{\calI} = \emptyset$
and $R^{\calI} = \{ (a,b)\}$ is a model of
$\M^{\chicq}$. This interpretation, however, 
cannot be extended to a model of $\ax_3$ (and hence of $\T$) without reinterpreting
$\A$, $\roleR$ or $\D$. 

The main insight behind locality and reachability
modules
is to ensure that each model of the module can be extended
to a model of $\T$ in a uniform way. 
Specifically,
each model of a $\top\!\bot\!^*$-locality or
$\top\!\bot\!^*$-reachability module can be
extended to a model of $\T$ by interpreting all other
predicates $\A$ as either $\emptyset$ or
$(\Delta^{\calI})^{n}$ with $n$ the arity of $\A$.
Thus, $\M = \{\ax_1,\ax_2,\ax_3\}$ is
a $\eqm_{\S}$-module of $\Tex$ since all its models
can be extended
by interpreting $\E$, $\F$ and $\G$ as
the domain, $\H$ as empty, and $\roleS$ as the Cartesian 
product of the domain.
We can capture this idea in our framework by means of the following setting.
\begin{definition} \label{def:chistar}
  The setting $\upchistar=\langle \theta^{\mathsf{m}}, \abox_0^{\mathsf{m}}, \arel^{\mathsf{m}} \rangle$ is
  as follows:
  $\theta^{\mathsf{m}}$ maps  each existentially quantified variable  to the fresh constant
  $\ast$ and
  $\abox_0^\mathsf{m} = \arel^\mathsf{m} = \abox_0^{\mathsf{f}}$. 
\end{definition}
In our example, 
$\prog^{\chistar}\cup\abox_0^{\mathsf{m}}$ entails the
relevant facts
$\A(\ast)$, $\roleR(\ast,\ast)$ and $\D(\ast)$, and hence  
$\M^{\chistar} = \{\ax_1,\ax_2,\ax_3\}$.

To show that $\M^{\chistar}$
is a $\eqm_{\S}$-module we prove that
all models $\calI$ of $\M^{\chistar}$
can be extended to a model of $\T$
as follows:
\begin{inparaenum}[\it(i)]
\item predicates not occurring in the materialisation
of $\prog^{\chistar} \cup \abox_0^{\mathsf{m}}$ are interpreted as empty;
\item predicates in the support
of $\upchistar$ (and hence occurring in $\M^{\chistar}$)
are interpreted as in $\calI$; and
\item all other predicates $\A$ are interpreted as $(\Delta^{\calI})^{n}$ with $n$ the arity of $\A$. 
\end{inparaenum}
%
\begin{restatable}{theorem}{datalogSemModules} \label{thm:datalogSemModules}
  $\M^{\chistar}\eqm_{\S}\T$. 
\end{restatable}

\subsection{Modules for Ontology Classification}\label{sec:modules-class}

Module extraction has been 
exploited for 
optimising ontology classification~\cite{DBLP:conf/semweb/RomeroGH12,DBLP:conf/ore/TsarkovP12,DBLP:journals/jar/GrauHKS10}.
In this case, it is not only required that 
modules are implication inseparable from $\T$, but also that they
preserve all implications  $\A(\x) \to \B(\x)$ with 
$\A\in \S$ but $\B \notin \S$. This requirement
can be captured as given next.

\begin{definition}
  TBoxes $\T$ and $\T'$ are \emph{$\S$-classification inseparable}
  ($\T\eqc_{\S}\T'$) if for each $\fml$ of the form $\A(\x)\to\B(\x)$
  with $\A\in\S, $ and
  $\B\in\sig{\T\cup\T'}$ we have $\T\models\fml$ iff $\T'\models\fml$.
 \end{definition}
Classification inseparability is a stronger requirement
than implication inseparability. For
$\T = \{ \A(x) \to \B(x)\}$ and $\S = \{ \A \}$,
$\M = \emptyset$ is implication inseparable from $\T$, whereas
classification inseparability requires that $\M = \T$. 

Modular reasoners such as MORe and Chainsaw rely on
locality $\bot$-modules, which satisfy this requirement.
Each model of a $\bot$-module $\M$ can be extended
to a model of $\T$ by interpreting all additional predicates as empty, which
is not possible if $\A \in \S$ and $\T$ entails $\A(x) \to \B(x)$ but $\M$ does not.
We can cast $\bot$-modules in our framework with
the following setting, which extends $\upchistar$
in Def.~\ref{def:chistar} by also considering
as relevant facts involving predicates not in $\S$.
\begin{definition} \label{def:chibot}
  The setting $\upchibot=\langle \theta^{\mathsf{b}}, \abox_0^{\mathsf{b}}, \arel^{\mathsf{b}} \rangle$ is as
  follows: $\theta^{\mathsf{b}} = \theta^{\mathsf{m}}$, 
  $\abox_0^{\mathsf{b}} = \abox_0^{\mathsf{m}}$, and 
  $\arel$ consists of all facts $\A(\ast, \ldots, \ast)$ where
  $\A\in\sig{\T}$.\qedhere
\end{definition}
The use of $\bot$-modules is, however, 
stricter than is needed for ontology classification.
For instance, if we consider
$\T = \Tex$ and $\S = \{\A\}$ we have
that $\M^{\chibot}$ contains all rules
$\ax_1$--$\ax_6$, but since $\A$ does not
have any subsumers in $\Tex$ the empty TBox
is already classification inseparable from $\Tex$.

The following module setting extends
$\upchiimp$ in Def.\ \ref{def:chiimp}
to ensure classification inseparability.
As in the case of $\chibot$ in Def.~\ref{def:chibot}
the only required modification is to
also consider as relevant facts 
involving predicates outside $\S$.

\begin{definition} \label{def:chicls}
 Setting $\upchicls= (\theta^{\mathsf{c}}, \abox_0^{\mathsf{c}}, \arel^{\mathsf{c}})$
  is as follows:
  $\theta^{\mathsf{c}} = \theta^{\mathsf{i}}$,
  $\abox_0^{\mathsf{c}} = \abox_0^{\mathsf{i}}$,
  and $\arel^{\mathsf{c}}$ consists of all facts
  $\B(c_{\A}^1, \ldots, c_{\A}^n)$ s.t.\ 
  $\A \neq \B$ are $n$-ary predicates, $A \in \S$ and
  $\B\in\sig{\T}$.
\end{definition}
Indeed, if we consider again $\T = \Tex$ and $\S = \{\A\}$, the module
for $\upchicls$ is empty, as desired.

\begin{restatable}{theorem}{datalogClassifModules}
  \label{thm:datalogClassifModules}
  $\M^{\chicls}\eqc_\S\T$. 
\end{restatable}

\subsection{Additional Properties of Modules}

Although the essential property of a module $\M$ is
that it captures all relevant $\S$-consequences of $\T$,
in some applications it is
desirable that modules satisfy additional requirements.
  
In ontology reuse scenarios, it is sometimes desirable that a module
$\M$ does not ``leave any relevant information behind'', in the sense that
$\T \setminus \M$ does not
entail any relevant $\S$-consequence---a property 
referred to as \emph{depletingness}
\cite{DBLP:journals/ai/KontchakovWZ10}.
\begin{definition}
  Let $\equiv_{\S}^z$ be an inseparability relation.
  A $\equiv_{\S}^z$-module $\M$ of $\T$ is \emph{depleting} if
  $\T\setminus\M\equiv_{\S}^z\emptyset$.
\end{definition}
Note that not all modules are depleting: for some relevant
$\S$-entailment $\varphi$ it may be that $\M \models \varphi$
(as required by the definition of module), but also that
$(\T \setminus \M) \models \varphi$, in which case $\M$ is not depleting.
The following theorem establishes 
that all modules defined in Section \ref{sec:all-modules}
are depleting.
\begin{restatable}{theorem}{depletingness}
  $\M^{\chi_z}$ is depleting for each $z\in\set{\mathsf{m},\mathsf{q},\mathsf{f},\mathsf{i},\mathsf{c}}$. 
\end{restatable}

Another common application of modules is to optimise the
computation of justifications: minimal subsets of
a TBox that are sufficient to entail a given formula
\cite{DBLP:conf/semweb/KalyanpurPHS07,DBLP:conf/aswc/SuntisrivarapornQJH08}.
\begin{definition}
  Let $\T\models\fml$. A \emph{justification
    for $\fml$ in $\T$} is a minimal subset $\T'\subseteq\T$ such that
  $\T'\models\fml$.
\end{definition}
Justifications are displayed in ontology development
platforms as explanations of why an entailment holds, and 
tools typically compute all of them. Extracting justifications
is a computationally intensive task, and locality-based modules have
been used to reduce the size of the problem: if $\T'$
is a justification of $\varphi$ in $\T$, then $\T'$ is contained in
a locality module of $\T$ for  $\S = \sig{\varphi}$. 
Our modules are
also justification-preserving, and we can adjust
our modules depending on what kind of first-order sentence $\varphi$ is.
\begin{restatable}{theorem}{presJustifications} \label{thm:presJustifications} ~
 Let $\T'$ be a justification for a first-order sentence $\varphi$ in $\T$ and let
 $\sig{\varphi} \subseteq \S$. Then, $\T' \subseteq \M^{\chistar}$. Additionally,
 the following properties hold:
 \begin{inparaenum}[\it (i)]
 \item if $\varphi$ is a rule, then $\T' \subseteq \M^{\chicq}$;
 \item if $\varphi$ is datalog, then $\T' \subseteq \M^{\chidd}$; and
 \item if $\varphi$ 
   is of the form $\A(\x)\to\B(\x)$, then $\T' \subseteq
   \M^{\chiimp}$; finally, if $\varphi$ satisfies
   $A \in \S, B \in \sig{\T}$, then
   $\T' \subseteq \M^{\chicls}$.
 \end{inparaenum}
\end{restatable}

\subsection{Complexity of Module Extraction}
\label{sec:complexity}
 
We conclude this section by showing that our modules can be 
efficiently computed in most practically relevant cases. 

\begin{restatable}{theorem}{complexity} \label{thm:complexity} %
  Let $m$ be a non-negative integer and $L$ a class of TBoxes s.t.\ 
  each rule in a TBox from $L$ has at most $m$
  distinct universally quantified variables.  The following problem is
  tractable: given $z \in \{\mathsf{q}, \mathsf{f}, \mathsf{i}, \mathsf{c}\}$, $\T\in L$, and 
  $\ax\in\T$, decide whether \mbox{$\ax\in\M^{\chi_z}$}.
 The problem is solvable in polynomial time for arbitrary classes $L$ of TBoxes if
 $z = \mathsf{m}$. 
\end{restatable}

We now provide a proof sketch for this result.
Checking whether a datalog program $\prog$ and an ABox $\abox$
entail a fact
is feasible in  $\mathcal{O}(\vert \prog \vert \cdot n^v)$, with 
$n$ the number 
of constants in $\prog \cup \abox$ and $v$ the maximum
number of variables in a rule from $\prog$ \cite{DBLP:journals/csur/DantsinEGV01}.
Thus, although datalog reasoning is exponential in the size of $v$
(and hence of $\prog$), it
is tractable if $v$ is bounded by a constant.

Given arbitrary $\T$ and $\S$, and for $z \in \{\mathsf{m}, \mathsf{q}, \mathsf{f}, \mathsf{i}, \mathsf{c}\}$,
the datalog program $\prog^{\chi_z}$ 
can be computed in linear time in the size of $\vert \T \vert$.
The number of constants $n$ in $\upchi_z$ (and hence in
$\prog^{\chi_z} \cup \abox_0^z$) is
linearly bounded in $\vert \T \vert$, whereas the maximum number of 
variables $v$ coincides with the maximum number
of universally quantified variables in a rule from $\T$. 
As shown in \cite{DBLP:conf/aaai/ZhouNGH14}, 
computing the support of a fact in a datalog program is
no harder than fact entailment, and thus module extraction in our approach is
feasible in  $\mathcal{O}(\vert \T \vert \cdot n^v)$, and thus
tractable for ontology languages where rules have a bounded number of variables (as is the case for
most DLs). Finally, if $z = \mathsf{m}$ the setting $\upchistar$ involves 
a single constant $\ast$ and module extraction boils down to reasoning
in propositional datalog (a tractable problem regardless of $\T$).

\subsection{Module Containment and Optimality} \label{sec:module-containment}

Intuitively, the more expressive the language for which preservation
of consequences is required the larger modules need to be.
The following proposition shows 
that our modules are consistent with this 
intuition.
%
\begin{restatable}{proposition}{hierarchy} \label{prop:module-hierarchy}
  $\M^{\chiimp} \subseteq \M^{\chidd} \subseteq \M^{\chicq} \subseteq
  \M^{\chistar}\subseteq\M^{\chibot}$ and $\M^{\chiimp} \subseteq \M^{\chicls} \subseteq \M^{\chibot}$
\end{restatable}
As already discussed, these containment relations 
are strict for many $\T$ and $\S$.

We conclude this section by discussing whether 
each
$\upchi_z$ with $z \in \{\mathsf{q}, \mathsf{f}, \mathsf{i}, \mathsf{c}\}$
is optimal for its inseparability relation in the sense that there is no
setting that produces smaller modules.
To make optimality statements precise we need to 
consider families of module settings, that is, functions
that assign a module setting for each pair of $\T$ and $\S$.
\begin{definition}
  A \emph{setting family} is a function $\Psi$ that maps a TBox
  $\T$ and signature $\S$ to a \bla{} for $\T$ and $\S$. 
  We say that $\Psi$ is \emph{uniform} if for every $\S$ and
  pair of TBoxes $\T,\T'$ with the same number of
  existentially
  quantified variables $\Psi(\T,\S)=\Psi(\T',\S)$.
  Let $z\in\set{\mathsf{i},\mathsf{f},\mathsf{q},\mathsf{c}}$; then,
  $\Psi$ is \emph{$z$-admissible} if, for each
  $\T$ and $\S$, $\M^{\Psi(\T,\S)}$ is a $\equiv^z_\S$-module of $\T$. 
  Finally, $\Psi$ is \emph{$z$-optimal} if
  $\M^{\Psi(\T,\S)}\subseteq\M^{\Psi'(\T,\S)}$ for every $\T$, $\S$
  and every uniform $\Psi'$ that is $z$-admissible.
\end{definition}
Uniformity ensures that settings
do not depend on the specific shape of rules in $\T$, but rather only on $\S$
and the number of existentially quantified variables in $\T$.
In turn, admissibility ensures that each setting yields a module.
The (uniform and admissible) family $\Psi^z$
corresponding to each setting  $\upchi^z$
in Sections \ref{sec:all-modules} and \ref{sec:modules-class}
is defined in the obvious way: 
for each $\T$ and $\S$, $\Psi^z(\T,\S)$ is the
setting $\upchi^z$ for $\T$ and $\S$.

The next theorem shows that $\Psi^z$ is optimal for implication and
classification inseparability.
\begin{restatable}{theorem}{optimality} \label{thm:optimality} %
  $\Psi^z$ is $z$-optimal for $z\in\set{\mathsf{i},\mathsf{c}}$.
\end{restatable}
In contrast, $\Psi^{\mathsf{q}}$ and  $\Psi^{\mathsf{f}}$
are not optimal. To see this, let
$\T = \{ \A(x) \to \B(x), \B(x) \to \A(x)\}$ and $\S = \{\A\}$. 
The empty TBox is fact inseparable from $\T$ since the only $\S$-consequence
of $\T$ is the tautology $\A(x) \to \A(x)$. However, $\M^{\chidd} = \T$ since 
fact $\A(a)$ is in $\arel^{\mathsf{f}}$ and its support is included in the module.
We can provide a family of settings that distinguishes
tautological from non-tautological inferences (see 
appendix);
however, this family yields settings of exponential size in $\vert \T \vert$, which is
undesirable in practice.




\section{Proof of Concept Evaluation}

We have implemented a prototype system for module extraction that
uses RDFox for datalog materialisation \cite{DBLP:conf/aaai/MotikNPHO14}. 
Additionally, the ontology reasoner PAGOdA \cite{DBLP:conf/aaai/ZhouNGH14} 
provides functionality for computing the support of an entailed
fact in datalog, which we have adapted for computing modules.
We have evaluated our system on representative ontologies, including
SNOMED (SCT), Fly Anatomy (FLY), the Gene Ontology (GO) and BioModels (BM).%
\footnote{The ontologies used in our tests are available for download at 
\texttt{http://www.cs.ox.ac.uk/isg/ontologies/UID/} under IDs 794 (FLY), 795 (SCT),
796 (GO) and 797 (BM).} 
SCT is expressed in the EL profile of OWL 2, whereas
FLY, GO and BM require expressive DLs (Horn-$\mathcal{SRI}$, $\mathcal{SHIQ}$ and
$\mathcal{SRIQ}$, respectively). 
We have normalised all ontologies 
to make axioms equivalent to rules.

We compared the size of our modules with the locality-based modules 
computed using the OWL API.
We have followed the experimental 
methodology from~\cite{VescovoKPS0T13} where two kinds of 
signatures are considered: \emph{genuine signatures} corresponding to
the signature of individual axioms, and 
\emph{random signatures} with a given probability for a symbol to be
included.
For each type of signature and ontology, we took a sample of
400 runs and
averaged module sizes. For random signatures we 
considered a probability of $1/1000$.
All experiments have been performed on a 
server with two Intel Xeon E5-2643 processors and 90GB of allocated RAM,
running RDFox on 16 threads.

\begin{table}[t]
  \centering
  \begin{tabular}{@{}c@{\,}|@{\;}r@{\;}|@{\;}r@{\;}|@{\;}r@{\;}|@{\;}r@{\;}|@{\;}r@{\;}|@{\;}r@{\;}|@{\;}r@{\;}|@{\;}r@{}}
     & \multicolumn{2}{@{}c@{\;}|@{\;}}{FLY} & \multicolumn{2}{@{}c@{\;}|@{\;}}{SCT} & \multicolumn{2}{@{}c@{\;}|@{\;}}{GO} & \multicolumn{2}{@{}c@{\;}}{BM}\\
    \hline
    rules & \multicolumn{2}{@{}c@{\;}|@{\;}}{19,830} & \multicolumn{2}{@{}c@{\;}|@{\;}}{112,833} & \multicolumn{2}{@{}c@{\;}|@{\;}}{145,083} & \multicolumn{2}{@{}c@{\;}}{462,120}\\
    \hline
    & \multicolumn{1}{@{}c@{\;}|@{\;}}{gen} & \multicolumn{1}{@{}c@{\;}|@{\;}}{rnd} & \multicolumn{1}{@{}c@{\;}|@{\;}}{gen} & \multicolumn{1}{@{}c@{\;}|@{\;}}{rnd} & \multicolumn{1}{@{}c@{\;}|@{\;}}{gen} & \multicolumn{1}{@{}c@{\;}|@{\;}}{rnd} & \multicolumn{1}{@{}c@{\;}|@{\;}}{gen} & \multicolumn{1}{@{}c@{\;}}{rnd} \\
    \hline
    $\bot, \upchibot$ & 242 & 847 & 242 & 5,196 & 1,461 & 12,801 & 1,010 & 64,320\\
    $\upchicls$ & 112 & 446 & 230 & 3,500 & 309 & 3,990 & 285 & 14,273\\
    \hline
    $\top\!\bot\!^*$ & 219 & 796 & 233 & 5,182 & 1,437 & 12,747 & 963 & 62,897\\
    $\upchistar$ & 215 & 789 & 233 & 5,182 & 1,431 & 12,724 & 955 & 62,286\\
    $\upchicq$ & 109 & 480 & 123 & 2,329 & 267 & 4,146 & 447 & 16,905\\
    $\upchidd$ & 76 & 476 & 24 & 2,258 & 162 & 4,142 & 259 & 14,043\\
    $\upchiimp$ & 8 & 7 & 15 & 235 & 103 & 2,429 & 105 & 4,107\\
    \hline
    $|\S|$ & 2.7 & 7.8 & 2.7 & 41.9 & 2.4 & 56.6 & 2.5 & 210.5
  \end{tabular}
  \caption{Results for genuine and random signatures $\S$}
  \label{tab:all}
\end{table}

Table~\ref{tab:all} 
summarises our results.
We compared $\bot$-modules with the modules for
$\upchicls$ (Section \ref{sec:modules-class})
and $\top\!\bot\!^*$-modules with those for $\upchistar$, $\upchicq$, $\upchidd$, and
$\upchiimp$ (Section \ref{sec:all-modules}).
We can see that module size consistently decreases as
we consider weaker inseparability relations. 
In particular, the modules 
for $\upchicls$ can be $4$ times smaller than $\bot$-modules.
%
%
The difference between $\top\!\bot\!^*$-modules and $\upchiimp$ modules is even bigger, 
especially in the case of FLY. In fact, $\upchiimp$ modules
are sometimes empty, which is not surprising since two predicates in a large
ontology are unlikely to be in an implication relationship.
%
%
Also note that
our modules for semantic inseparability slightly improve on
$\top\!\bot\!^*$-modules.
Finally, recall that our modules may not be
minimal for their inseparability relation. Since
techniques for extracting minimal modules
are available only for model inseparability, and for 
restricted languages,
we could not assess how close our modules are
to minimal ones and hence the quality of our approximation.

Computation times were comparable for all settings $\upchi^z$ with
times being slightly higher for $\upchiimp$ and $\upchicls$ as they
involved a larger number of constants.  Furthermore, extraction times
were comparable to locality-based modules for genuine signatures with
average times of $0.5s$ for FLY, $0.9s$ for SCT, $4.2$ for GO and $5s$
for BM.

\section{Conclusion and Future Work}

We have proposed a novel approach to module extraction by
exploiting off-the-shelf datalog reasoners, which allows us
to efficiently compute approximations of minimal modules
for different inseparability relations. Our results
open the door to significant improvements in
common applications of modules, such as
computation of justifications, modular and incremental reasoning
and ontology reuse, which currently rely mostly on locality-based
modules. 

Our approach is novel, and we see many interesting open problems.
For example, the issue
of optimality requires further investigation. 
Furthermore, it would be interesting to integrate our 
extraction techniques
in existing modular reasoners as well as in systems for
justification extraction.


\section*{Acknowledgements}
This work was supported by the Royal Society, the EPSRC projects MaSI$^3$, Score! and DBOnto, 
and by the EU FP7 project \mbox{OPTIQUE}.

\bibliography{bibliography}
\bibliographystyle{aaai}
 
 \clearpage
 \onecolumn
 \appendix

\section{Inseparability Relations}

We start by giving an alternative characterization of $\S$-query and $\S$-fact inseparability 
that will allow us to prove our results in a more uniform and clear way.

\begin{proposition} \label{prop:fq-insep-charact-rules}
TBoxes $\T$ and $\T'$ are
\begin{enumerate}
\item $\S$-query inseparable iff $\T\models \rrule \Leftrightarrow \T'\models \rrule$ holds for every rule $\rrule$ over $\S$;
\item $\S$-fact inseparable iff $\T\models \rrule \Leftrightarrow \T'\models \rrule$ holds for every datalog rule $\rrule$ over $\S$.
\end{enumerate}
\end{proposition}
\begin{proof}
 It suffices to observe that, for every TBox $\T$ and every rule $\rrule=\fml\to\fmm$ over $\S$, 
 $\T\models\rrule$ iff $\T\cup\mset{\fct\s}{\fct\in\fml}\models\fmm\s$, with $\s$ a substitution 
 mapping all free variables in $\rrule$ to fresh, pairwise distinct constants.
\end{proof}

\begin{proposition} \label{prop:insep-rel-hierarchy}
${\eqm}\subsetneq{\eqq}\subsetneq{\eqf}\subsetneq{\eqi}$.
\end{proposition}
\begin{proof}
  The inclusion ${\eqq}\subseteq{\eqf}$ is immediate by definition
  while ${\eqf}\subseteq{\eqi}$ follows by
  Proposition~\ref{prop:fq-insep-charact-rules}. The inclusion
  ${\eqm}\subseteq{\eqq}$ follows since
  $\ax^{\calI}=\ax^{\calI|_{\S}}=\ax^{\calI'|_{\S}}=\ax^{\calI'}$ for
  every rule $\ax$ over $\S$ whenever $\calI$ and $\calI'$ coincide on
  $\S$.

  To show strictness of the inclusions, we can w.l.o.g. restrict
  ourselves to the signature $\S=\set{\roleQ,\bot}$ where $\roleQ$ is
  a unary predicate (if $\S$ contains more symbols, one can consider
  $\T$ such that $\sig{\T}\cap\S\subseteq\set{\roleQ}$; adapting the
  argument to higher arities for $\roleQ$ is also straightforward;
  finally, the presence of $\bot$ in $\S$ is not relevant for the
  proof).

  For ${\eqm}\subsetneq{\eqq}$, suppose $\T=\set{\top(x)\to\exists
  y.[\roleR(x,y)\land\A(y)],\top(x)\to\exists y.[\roleR(x,y)\land\B(y)],\A(x)\land\B(x)\to\roleQ(x)}$. 
  Then $\T\eqa\emptyset$.
  However, $\T\not\eqm\emptyset$ since, for any interpretation $\calI$ with a singleton domain 
  such that $\roleQ^{\calI} = \emptyset$, $\calI$ cannot be turned into a model of $\T$ 
  without changing the interpretation of $\roleQ$.

  For ${\eqq}\subsetneq{\eqf}$, suppose $\T=\set{\top(x)\to\exists y.[\roleR(x,y)\land\roleQ(y)]}$. 
  Then $\T\eqf\emptyset$ but $\T\not\eqq\emptyset$ since $\T\models\exists x.\roleQ(x)$ 
  while $\emptyset\not\models\exists x.\roleQ(x)$.

%

  For ${\eqf}\subsetneq{\eqi}$, suppose $\T=\set{r}$ where
  $r=\roleQ(a)\land\roleQ(b)\to\roleQ(c)$. Then $\T\eqi\emptyset$ but
  $\T\not\eqd\emptyset$ since $\T\models r$ while $\emptyset\not\models r$.
  \end{proof}

\section{Deductive Inseparability}


Theorems \ref{thm:datalogAtomicSubsumptionModules}, \ref{thm:disjDatalogModules},
\ref{thm:datalogCQModules},
\ref{thm:datalogClassifModules}
are all shown by a similar argument, which we present next.

\subsubsection{Hyperresolution}
Given $\ax = \fml(\x)\to \exists \y. [\bigvee_{i=1}^n
\fmm(\x,\y)]\in\T$\/ we denote with $\sk(\ax)$ the result of applying
standard Skolemisation to $\ax$---which replaces, for each $y\in\y$,
all occurrences of $y$ in $\ax$ by $f_y(\x)$, where $f_y$ is a fresh
function symbol unique for $y$.  Given a substitution $\theta$ mapping
existentially quantified variables in $\T$ to constants and a
Skolemised formula $\fml$, we write $\Gamma_\theta(\fml)$ for the
formula obtained from $\fml$ by replacing every occurrence of a
functional term $f_y(\t)$ by the constant $y\theta$.

By distributing disjunctions over conjunctions in the head of
$\sk(\ax)$ we obtain a rule of the form $\fml \to \bigwedge_{j=1}^m
\fmm'_j$ where each $\fmm'_j$ is a disjunction of atoms.  We denote
with $\cnf(\ax)$ the set $\mset{\fml \to \fmm'_j}{1\leq j\leq m}$ and
extend this notation in the natural way to $\cnf(\T) =
\bigcup_{\ax\in\T}\cnf(\ax)$.  We call $\cnf(\T)$ a \emph{CNF TBox}
and each $\srule\in\cnf(\T)$ a \emph{CNF rule}.
Clearly, $\cnf(\T)\models\T$, and hence $\T\cup\abox\models \fml'$
implies $\cnf(\T)\cup\abox\models \fml'$ for every $\abox$ and
$\fml'$.

Let $\fml$ be a disjunction of facts, $\abox$ an ABox,
and $\srule=\bigwedge_{i=1}^n\fct'_i\to\fmm\in\cnf(\T)$.
A formula $\fml$ is a \emph{hyperresolvent} of $\srule$ and ground disjunctions 
\mbox{$\fct_1\vee \fmm_1,\dots,\fct_n\vee \fmm_n$} (with each $\fmm_i$ potentially empty)
if $\fml = \bigvee_{i=1}^n\fmm_i \vee \fmm\s$  with $\s$
a MGU of $\fct_i,\fct'_i$ for each $1 \leq i \leq n$.
Let $\cT$ be a CNF TBox. A hyperresolution proof (or simply a \emph{proof}) 
of $\fml$ in $\cT\cup \abox$ is a pair $\rho = (T, \lambda)$ where $T$ is a tree, 
$\lambda$ is a mapping from nodes in $T$ to disjunctions of  facts,
and from edges in $T$ to CNF rules in $\cT$, such that 
for every node $v$ the following properties hold:
\begin{enumerate}
\item $\lambda(v) = \fml$ if $v$ is the root of $T$;
\item $\lambda(v) \in \abox$ if $v$ is a leaf in $T$;
and 
\item if $v$ has children
$w_1,\dots,w_n$ then each edge from $v$
to $w_i$ is labelled by the same CNF rule $\srule$
and  $\lambda(v)$ is a hyperresolvent of $\srule$
and $\lambda(w_1), \dots, \lambda(w_n)$.
\end{enumerate}
If there exists a proof of $\fml$ in $\cT\cup\abox$ we write $\cT\cup\abox\vdash\fml$.
The \emph{support} of $\fml$ is the set of CNF rules occurring in some proof of $\fml$ in $\cT\cup\abox$.

Hyperresolution is sound (if $\cT\cup\abox\vdash\fml$ then $\cT\cup\abox\models\fml$)
and complete in the following sense:
if $\cT\cup\abox\models \fml$ then there exists $\fmm\subseteq \fml$ such that 
$\cT\cup\abox \vdash \fmm$.

\vspace{4mm}

Given a \bla{} $\upchi$ and $\rrule\in\T$, we denote with $\Xi^{\chi}(\rrule)$ the set of datalog rules in 
$\prog^{\chi}$ corresponding to $\rrule$, as described in Definition \ref{def:bla}.
The following auxiliary results provide the basis for 
correctness of our approach to module extraction.

\begin{restatable}{lemma}{hyperresolutionProofsToDatalogProofs}
\label{lemma:HyperresolutionProofsToDatalogProofs}
Let $\upchi = \langle \theta, \abox_0, \arel \rangle$ be a \bla{}.
Let $\N$ be the set of constants mentioned in $\upchi$. 
Let $\abox$ be a function-free ABox that only mentions constants that are
fresh w.r.t. $\T$ and $\N$.
Let $\nu$ be a mapping from constants in $\abox$ to $\N$ such that 
$\abox\nu\subseteq \abox_0$.
Let $\fml$ be a disjunction of facts and $\rho = (T, \lambda)$ a proof of 
$\fml$ in $\cnf(\T)\cup\abox$. 
The following properties hold:
\begin{enumerate}
\item $\prog^{\chi}\cup\abox_0 \vdash \Gamma_{\theta}(\fct\nu)$ for every $\fct\in\fml$.
\item For every $\ax\in\T$ such that $\rho$ mentions some $\srule\in \cnf(\ax)$
there exists $\fct\in \fml\cup\{\bot\}$ and a proof of $\Gamma_{\theta}(\fct\nu)$ 
in $\prog^{\chi}\cup\abox_0$ that mentions some rule in $\Xi^{\chi}(\ax)$. 
\end{enumerate}
\end{restatable}
\begin{proof}
We reason by induction on the depth $d$ of $\rho$.
\begin{itemize}
\item [$d=0$]\ \\
In this case $\fml$ must be a fact in $\abox$. Since $\abox$ is function-free by assumption 
we have $\Gamma_{\theta}(\fml\nu) = \fml\nu$, and since $\abox\nu\subseteq \abox_0$
we have $\fml\nu\in\abox_0$. Therefore, there exists a trivial proof of $\Gamma_{\theta}(\fml\nu)$ 
in $\prog^{\chi}\cup\abox_0$ and property $1$ is satisfied. Furthermore, if the depth 
of $\rho$ is $0$ then there cannot be any rules in its support, so property $2$ is 
trivially satisfied as well. 

\item [$d>0$]\ \\
Let $v$ be the root of $T$ and $w_1,\dots,w_n$ the children of $v$. 
Then it must be
\begin{itemize}
\item $\lambda(w_i) = \fcu_i \vee \fmm_i$ for each $1\leq i \leq n$;
\item $\lambda(v,w_i) = \srule$ for each $1\leq i \leq n$ with
$\srule \in \cnf(\T)$ of the form $\bigwedge_{i=1}^n \fcu'_i\to \fml'$; and
\item $\fml = \bigvee_{i=1}^n\fmm_i\vee\fml'\s$ with $\s$ a MGU of $\fcu_i, \fcu'_i$ for each $1\leq i \leq n$.
\end{itemize}
Consider $\fct\in\fml$. To show property $1$ we need to find
a proof of $\Gamma_{\theta}(\fct\nu)$ in $\prog^{\chi}\cup\abox_0$. 
If $\fct\in\fmm_i$ then by i.h. we can find such a proof.
If  $\fct\in\fml'\s$ then it must be $\fct = \fct'\s$ for some $\fct'\in\fml'$ and, 
by definition of $\prog^{\chi}$, $\cnf(\T)$, and $\Gamma_{\theta}$, $\srule \in \cnf(\T)$ implies 
$\bigwedge_{i=1}^n \fcu'_i\to \Gamma_{\theta}(\fct')\in\prog^{\chi}$.
Since $\s$ is a MGU of $\fcu_i, \fcu'_i$ (with $\fcu_i = \fcu'_i\s$) for each $1\leq i \leq n$, 
$\s\nu$ must be a MGU of $\fcu_i\nu, \fcu'_i$ (with $\fcu_i\nu = \fcu'_i\s\nu$) for each $1\leq i \leq n$;
furthermore, since $\fcu'_i$ is necessarily function-free, 
it is $\Gamma_{\theta}(\fcu'_i\s\nu) = \fcu'_i\s\nu$, 
and thus $\Gamma_{\theta}(\fcu_i\nu) = \fcu'_i\s\nu$ and
$\s\nu$ is also a MGU of $\Gamma_{\theta}(\fcu_i\nu), \fcu'_i$ for each $1\leq i \leq n$.
By i.h. we have a proof in $\prog^{\chi}\cup\abox_0$ of $\Gamma_{\theta}(\fcu_i\nu)$ 
for each $1\leq i \leq n$;
it is easy to see that 
$\Gamma_{\theta}(\fct')\s\nu = \Gamma_{\theta}(\fct'\s\nu) = \Gamma_{\theta}(\fct\nu)$,
so combining these proofs with rule 
$\bigwedge_{i=1}^n \fcu'_i\to \Gamma_{\theta}(\fct')$ yields a proof of
$\Gamma_{\theta}(\fct\nu)$ in $\prog^{\chi}\cup\abox_0$.

Now consider $\ax\in\T$ such that $\rho$ mentions some $\srule'\in\cnf(\ax)$.
To show property 2 we need to find $\fct\in\fml\cup\set{\bot}$ and a proof of 
$\Gamma_{\theta}(\fct\nu)$ that mentions some rule in $\Xi^{\chi}(\rrule)$.
Assume first that $\srule' = \srule$.
If $\fml'=\emptyset$ then it must be 
$\cnf(\ax) = \set{\bigwedge_{i=1}^n \fcu'_i\to \bot}\subseteq\Xi^{\chi}(\rrule)$ and, as before,
we can combine this rule with proofs in $\prog^{\chi}\cup\abox_0$ of the $\Gamma_{\theta}(\fcu_i\nu)$ 
to obtain a proof of $\bot$ in $\prog^{\chi}\cup\abox_0$.
If $\fml'\neq \emptyset$ then it must be 
$\mset{\bigwedge_{i=1}^n \fcu'_i\to \Gamma_{\theta}(\fct')}{\fct'\in\fml'}\subseteq\Xi^{\chi}(\ax)$.
Since $\fml = \bigvee_{i=1}^n\fmm_i\vee\fml'\s$, for each $\fct'\in\fml'$ it is $\fct'\s\in\fml$ 
and, as we just saw, we can construct a proof of $\Gamma_{\theta}(\fct'\s\nu)$ that mentions
$\bigwedge_{i=1}^n \fcu'_i\to \Gamma_{\theta}(\fct')$.
Finally, assume that $\srule' \neq \srule$.
Then there must be some $i\in\set{1,\dots,n}$ such that $\srule'$ is mentioned by the proof 
$\rho_i$ of $\fcu_i\vee\fmm_i$ that is embedded in $\rho$.
Since $\rho_i$ is of depth $<d$, by i.h. there must be $\fcu''\in\fcu_i\vee\fmm_i$
and a proof $\rho''$ of $\Gamma_{\theta}(\fcu''\nu)$ in $\prog^{\chi}\cup\abox_0$ that mentions
some rule in $\Xi^{\chi}(\ax)$. If $\fcu'' \in \fmm_i$ then $\fcu''\in \fml$ already; if $\fcu'' = \fcu_i$ 
then, as before, for any $\fct\in\fml$ we can construct a proof of $\Gamma_{\theta}(\fct\nu)$ 
in $\prog^{\chi}\cup\abox_0$ such that $\rho''$ is embedded in it.
\qedhere
\end{itemize}
\end{proof}

\begin{restatable}{proposition}{generalConsequencePreservation}
\label{prop:generalConsequencePreservation}
Let $\rrule=\fml(\x)\to\fmm(\x)$ with $\fml$ a conjunction and $\fmm$ a disjunction
of atoms. Let $\upchi = \langle \theta, \abox_0, \arel\rangle$ be  a \bla{} satisfying
$\{\bot\} \subseteq \arel$ and such that for every substitution $\s$ mapping all variables in 
$r$ to pairwise distinct constants not in $\T$ there exists a mapping $\nu_{\s}$ with 
$\fml\s\nu_{\s}\subseteq\abox_0$ and $\fmm\s\nu_{\s}\subseteq\arel$. 
Then
\begin{enumerate}
\item $\T\models \rrule$ iff $\M^{\chi}\models \rrule$;

\item if $\T'\subseteq \T$ is a justification for $\rrule$ in $\T$ then $\T'\subseteq\M^{\chi}$;

\item $\T\setminus \M^{\chi}\models \rrule$ iff $\emptyset\models \rrule$.
\end{enumerate}
\end{restatable}
\begin{proof}\ 

\begin{enumerate}
\item By monotonicity, it is immediate that $\T\models \rrule$ if $\M^{\chi}\models \rrule$.

Suppose $\T\models r$ and 
let $\s$ be a substitution mapping all variables in $\rrule$ to fresh, pairwise 
distinct constants. Then we have that $\T\cup\mset{\fct\s}{\fct\in\fml}\models \fmm\s$, which
implies $\cnf(\T)\cup\mset{\fct\s}{\fct\in\fml}\models \fmm\s$ and by
completeness of hyperresolution $\cnf(\T)\cup\mset{\fct\s}{\fct\in\fml}\vdash\fmm'$
for some $\fmm'\subseteq\fmm\s$.
Since $\mset{\fct\s\nu_{\s}}{\fct\in\fml}\subseteq \abox_0$, 
by Lemma \ref{lemma:HyperresolutionProofsToDatalogProofs} we have that for each 
$\srule\in\T$ such that some $\prule\in\cnf(\srule)$ supports 
$\fmm'$ in $\cnf(\T)\cup\mset{\fct\s}{\fct\in\fml}$
there exists $\fct\in\Gamma_{\theta}(\fmm'\nu_{\s})\cup\set{\bot}$
that is supported in $\prog^{\chi}\cup\abox_0$ by some rule from $\Xi^{\chi}(\srule)$.
By assumption, $\bot\in\arel$; also, $\fmm'$ is function-free so 
$\Gamma_{\theta}(\fmm'\nu_{\s}) = \fmm'\nu_{\s}$, and hence,
since $\fmm\s\nu_{\s}\subseteq \arel$, 
we have that $\fmm'\nu_{\s}\subseteq \arel$ and also $\fct\in\arel$.
In either case we have $\srule\in\M^{\chi}$ and consequently $\M^{\chi}\models \rrule$.

\item Let $\T'\subseteq \T$ be a justification for $\rrule$ in $\T$. 
As before, if $\s$ is a ground substitution for $\rrule$ mapping variables in $\rrule$ to fresh, 
pairwise distinct constants, then 
$\cnf(\T')\cup\mset{\fct\s}{\fct\in\fml}\vdash\fmm'$ 
for some $\fmm'\subseteq\fmm\s$.
In fact, by minimality of justifications, for each $\srule\in\T'$ some 
$\prule\in \cnf(\srule)$ must be in the support of 
some $\fmm'\subseteq \fmm\s$
in $\cnf(\T')\cup \mset{\fct\s}{\fct\in\fml}$.
As before, by Lemma \ref{lemma:HyperresolutionProofsToDatalogProofs}, this implies
$\srule\in\M^{\chi}$.

\item By monotonicity, it is immediate that 
$\T\setminus \M^{\chi}\models r$ if $\emptyset\models r$.

Suppose $\T\setminus\M^{\chi}\models \rrule$ and let $\T'$ be a justification for $\rrule$ in 
$\T\setminus\M^{\chi}$. Then $\T'$ is also a justification for $\rrule$ in $\T$, and, as we just 
proved, $\T'\subseteq \M^{\chi}$. This implies that $\T' = \emptyset$, and therefore
$\emptyset\models \rrule$.\qedhere
\end{enumerate}
\end{proof}

\begin{proposition}
\label{prop:cqPreservation}
Let $\ax = \fml(\x)\to \exists \y.[\bigvee_{i=1}^n \fmm_i(\x,\y)]$ be a rule. 
Let $\upchi = \langle \theta, \abox_0, \arel \rangle$ be a \bla{} satisfying
\begin{itemize}
\item $\{\bot\}\subseteq \arel$ and also $\fmm\s\subseteq \arel$ for every substitution 
$\s$ mapping all variables in $\rrule$ to constants in $\upchi$;
\item for every substitution $\s$ mapping all variables in $\rrule$ to pairwise distinct constants not in $\T$ there exists 
a mapping $\nu_{\s}$ such that $\fml\s\nu_{\s}\subseteq\abox_0$. 
\end{itemize}
Then
\begin{enumerate}
\item $\T\models \ax$ iff $\M^{\chi}\models \ax$;
\item if $\T'\subseteq \T$ is a justification for $\ax$ in $\T$ then $\T'\subseteq\M^{\chi}$;
\item $\T\setminus\M^{\chi}\models \ax$ iff $\emptyset\models \ax$.
\end{enumerate}
\end{proposition}

\begin{proof}\
\begin{enumerate}
\item By monotonicity, it is immediate that $\T\cup\abox\models \ax$ if $\M^{\chi}\cup\abox\models \ax$.

Let $\roleQ$ be a fresh predicate and $\T_{\fmm\to\roleQ} = \mset{\fmm_i(\x,\y) \to \roleQ(\x)}{1\leq i \leq n}$.
Then 
\[\T\models \ax \text{\  iff \ } 
\T\cup \T_{\fmm\to\roleQ} \models \fml(\x)\to\roleQ(\x)
\text{\ \ \ \ and \ \ \ \ }
\M^{\chi}\models \ax \text{\  iff \ } 
\M^{\chi}\cup\T_{\fmm\to\roleQ}\models \fml(\x)\to\roleQ(\x)\]

Consider $\T' = \T\cup\T_{\fmm\to\roleQ} $ and $\S' = \S\cup\{\roleQ\}$.
Clearly, $\T'$ has the exact same existentially quantified variables as $\T$. Therefore
$\upchi' = \langle \theta, \abox_0, \arel'\rangle$ with 
\[\arel' = \mset{\roleQ(\x)\s}{\s \text{ is a substitution mapping all variables in } \x \text{ to constants in }\upchi}\]
is a \bla{} for $\T'$ and $\S'$
and by Proposition \ref{prop:generalConsequencePreservation} we have that $\T'\models \fml(\x)\to\roleQ(\x)$ iff
$\M^{\chi'}\models \fml(\x)\to\roleQ(\x)$. 

If we show that $\M^{\chi'}\setminus\T_{\fmm\to\roleQ}  \subseteq \M^{\chi}$
then, by monotonicity, we will be able to conclude that 
\[\M^{\chi'}\models \fml(\x)\to\roleQ(\x) \text{ implies }
\M^{\chi}\cup\T_{\fmm\to\roleQ} \models \fml(\x)\to\roleQ(\x)\]
and thus that $\T\models \ax$ implies $\M^{\chi}\models \ax$.

Let $\srule\in \M^{\chi'}\setminus\T_{\fmm\to\roleQ} $.
Some $\prule\in\Xi^{\chi'}(\srule) = \Xi^{\chi}(\srule)$
must be in the support of some $\roleQ(\x)\s\in\arel'$ in $\prog^{\chi'}\cup\abox_0$.
In particular, $\prule$ must be mentioned in some proof $\rho = (T,\lambda)$ of $\roleQ(\x)\s$ 
in $\prog^{\chi'}\cup\abox_0$.
Let $v$ be the root of $T$ and $w_1,\dots,w_m$ its children nodes, there must be some 
$\bigwedge_{j=1}^m\fct_j(\x,\y)\to \roleQ(\x)\in\T_{\fmm\to\roleQ} $ 
and a MGU $\s'$ of $\fct_j, \lambda(w_j)$ for each $1\leq j \leq m$.
Since $\srule\notin \T_{\fmm\to\roleQ}$, 
there must exist $j\in\set{1,\dots,m}$ such that $\prule$ is mentioned in the proof $\rho_j$ 
of $\lambda(w_j)$ in $\prog^{\chi'}\cup\abox_0$ that is embedded in $\rho$.
Furthermore, since $\roleQ$ does not occur in the body of any rule in 
$\prog^{\chi'} = \prog^{\chi}\cup\T_{\fmm\to\roleQ}$, all rules 
mentioned in $\rho_j$ must be in $\prog^{\chi}$ and thus $\rho_j$ is a proof of $\fct$ in
$\prog^{\chi}\cup\abox_0$. Since by assumption $\lambda(w_j) = \fct_j\s'\in\arel$, this 
implies $\srule\in\M^{\chi}$.

\item 
Let $\T''\subseteq \T$ be a justification for $\ax$ in $\T$.
As before,
$\T''\models \ax \text{\ \  implies\ \ } 
\T''\cup\T_{\fmm\to\roleQ}\models \fml(\x)\to\roleQ(\x)$
and in particular for any substitution $\s$ mapping variables in $\rrule$ to pairwise distinct constants
we have 
$\T''\cup\T_{\fmm\to\roleQ} \cup \mset{\fct\s}{\fct\in\fml}\models \roleQ(\x)\s$
and therefore
$\cnf(\T''\cup\T_{\fmm\to\roleQ}) \cup \mset{\fct\s}{\fct\in\fml}\vdash \roleQ(\x)\s$.
%
By minimality of justifications, for each $\srule\in\T''$ there must be some $\prule\in\cnf(\srule)$
in the support of $\roleQ(\x)\s$ in $\cnf(\T''\cup\T_{\fmm\to\roleQ})\cup\mset{\fct\s}{\fct\in\fml} $.
It is easy to see that $\prule$ must also be in the support of $\roleQ(\x)\s\nu_{\s}$ in 
$\cnf(\T''\cup\T_{\fmm\to\roleQ}) \cup \mset{\fct\s\nu_{\s}}{\fct\in\fml}$.
Since $\roleQ(\x)\s\nu_{\s}\subseteq \arel'$
and $\mset{\fct\s\nu_{\s}}{\fct\in\fml}\subseteq \abox_0$,
by Lemma \ref{lemma:HyperresolutionProofsToDatalogProofs}
we have  that $\srule\in\M^{\chi'}$. In particular, since $\srule\in\T''\subseteq \T$, 
it must be 
$\srule\in\M^{\chi'} \setminus \T_{\fmm\to\roleQ} \subseteq \M^{\chi}$.

\item Again by monotonicity, it is immediate that  $\T\setminus\M^{\chi}\cup\abox\models \rrule$ if $\abox\models \rrule$.
By a similar argument to the one given in Proposition \ref{prop:generalConsequencePreservation}, it follows from 2 
that any justification for $\ax$ in $\T\setminus\M^{\chi}$ must be empty and therefore 
if $\T\setminus \M^{\chi}\models \ax$ then $\emptyset\models \ax$.\qedhere
\end{enumerate}
\end{proof}

\datalogAtomicSubsumptionModules*
\begin{proof}
  Consider an arbitrary rule of the form $\A(\x)\to \B(\x)$ with
  $\A,\B\in\S$ and $\A\ne\B$ (if $\A=\B$ the rule is tautological).
  Since $\x$ is implicitly universally quantified, we can assume
  w.l.o.g.  that $\x = (x_1,\dots,x_n)$ with $x_1, \dots,x_n$ pairwise
  distinct. 
  Let $\s$ be a substitution mapping $x_1,\dots,x_n$, respectively, to
  $c_1,\dots,c_n$, pairwise distinct constants not in $\T$.  Now
  consider a mapping $\nu_{\s}$ such that $c_i\nu_{\s} =
  c^i_{\A}$. This mapping is well-defined because $c_1,\dots,c_n$ are
  pairwise distinct. By definition of $\upchiimp$, we have
  $\A(\x)\s\nu_{\s}\in\abox_0^{\mathsf{i}}$ and
  $\B(\x)\s\nu_{\s}\in\arel^{\mathsf{i}}$, and therefore, by
  Proposition \ref{prop:generalConsequencePreservation}, we have
  $\T\models \A(\x)\to\B(\x)$ iff $\M^{\chiimp}\models
  \A(\x)\to\B(\x)$.
\end{proof}

\disjDatalogModules*
\begin{proof}
  By Proposition \ref{prop:fq-insep-charact-rules} it suffices to show
  that for any datalog rule $r=\fml\to\fmm$ we have $\T\models r$ iff
  $\M^{\chidd}\models r$.  Let $\s$ be a substitution mapping all
  variables in $r$ to pairwise distinct constants not in $\T$.
  Consider a mapping $\nu^*$ such that $x\s\nu^* = *$ for each
  $x\in\x$. Clearly $\fml\s\nu^*\subseteq\abox_0^{\mathsf{f}}$ and
  $\fmm\s\nu^*\subseteq\arel^{\mathsf{f}}$, and therefore, by
  Proposition \ref{prop:generalConsequencePreservation}, we have
  $\T\models r$ iff $\M^{\chidd}\models r$.
\end{proof}

\datalogCQModules*
\begin{proof}
  By Proposition \ref{prop:fq-insep-charact-rules} it suffices to show
  that for any rule $r=\fml\to\fmm$ we have $\T\models r$ iff
  $\M^{\chicq}\models r$.  Let $\s$ be a substitution mapping all
  variables in $r$ to pairwise distinct constants not in $\T$.
  Given a mapping $\nu^*$ such that $x\s\nu^* = *$ for each
  $x\in\x$ it is clear that
  $\fml\s\nu^*\subseteq\abox_0^{\mathsf{q}}$.  It is also immediate
  that $\fmm\s'\subseteq\arel^{\mathsf{q}}$ for every substitution
  $\s'$ mapping all variables in $r$ to constants in $\upchicq$.
  Therefore, by Proposition \ref{prop:cqPreservation}, we have 
  $\T\models r$ iff $\M^{\chicq}\models r$.
\end{proof}

\datalogClassifModules*
\begin{proof}
Analogous to the proof of Theorem \ref{thm:datalogAtomicSubsumptionModules}.
\end{proof}

\section{Model Inseparability}

Given an ABox $\abox$ and a datalog program $\prog$, let $\prog(\abox)$ denote the
materialisation of $\prog\cup\abox$. 
Furthermore, given a \bla{} $\upchi$, let $\rel(\upchi)$
denote the support of $\upchi$.

\datalogSemModules*
\begin{proof}
Let $\calI$ be a model of $\M^{\chistar}$. We assume w.l.o.g that $\calI$ is defined over all of $\sig{\T}$. 
Consider the interpretation $\calJ$ over $\sig{\T}$ such that  $\Delta^{\calJ} = \Delta^{\calI}$ and 
\[\begin{array}{rcl}
 \A^{\calJ} & = & 
\left\lbrace\begin{array}{lll}
\A^{\calI} &  \text{ if } \A\in(\S\cup\sig{ 
\rel(\upchi)})\setminus \set{\bot}\\
\D^{\arity(\A)} & \text{ if } \A\in\sig{\prog^{\chistar}(\abox_0^{\mathsf{m}})}\setminus (\S\cup\sig{\rel(\upchi)})\\
\emptyset & \text{ otherwise}
\end{array}
\right.
\end{array}\]

Consider $\ax : \fml(\x) \to \exists \y. [\bigvee_{j=1}^m  \fmm_j(\x,\y)] \in \T$.
We will show that $\calJ\models \ax$.

Assume first $m=0$. Then $\Xi^{\chistar}(\ax) =\set{ \fml \to \bot}$.
If $r\in\M^{\chistar}$ then in particular $\sig{r}\subseteq \sig{\rel(\upchistar)}$, 
so $\calI$ and $\calJ$ agree over $\sig{\ax}$, and $\calJ\models \ax$.
If $r\notin\M^{\chistar}$ then, since $\bot\in\arel^{\chistar}$ and the only constant mentioned 
in $\prog^{\chistar}\cup\abox_0^{\mathsf{m}}$ is $*$, there must be $\fct\in\fml$ such that 
$\fct*\notin \prog^{\chistar}(\abox_0^{\mathsf{m}})$
(where, in an abuse of 
notation, $*$ denotes the substitution that maps all variables to $*$), and in particular 
$\sig{\fct}\not\subseteq \sig{\prog^{\chistar}(\abox_0^{\mathsf{m}})}$.
Since $\S\cup\sig{\rel(\upchistar)}\subseteq \sig{\prog^{\chistar}(\abox_0^{\mathsf{m}})}$,
this implies that for $\A\in\sig{\fct}$ it is $\A^{\calJ} = \emptyset$ and therefore
trivially $\calJ\models r$.

Assume now $m>0$ and let $\s$ be a substitution over all variables in $\ax$ such that 
$\calJ\models\fml\s$
(if no such substitution exists then 
trivially $\calJ\models \ax$). 
Since $\S\cup\sig{\rel(\upchi)}\subseteq \sig{\prog^{\chistar}(\abox_0^{\mathsf{m}})}$,
all predicates in $\fml$ must occur in $\prog^{\chistar}(\abox_0^{\mathsf{m}})$.
In particular it must be $\fct*\in\prog^{\chistar}(\abox_0^{\mathsf{m}})$ for every $\fct\in\fml$. 
This implies $\fcu*\in\prog^{\chistar}(\abox_0^{\mathsf{m}})$ for every $\fcu\in\bigcup_{j=1}^m\fmm_j$
and therefore for every predicate $\A$ in 
$\sig{\bigvee_{j=1}^m  \fmm_j}$ we have that either $\A^{\calJ} = \A^{\calI}$
or $\A^{\calJ} = \Delta^{\arity(\A)}$---in particular $\A^{\calI} \subseteq \A^{\calJ}$.
If $\A^{\calJ} = \Delta^{\arity(\A)}$ for every $\A\in\sig{\bigvee_{j=1}^m \fmm_j}$,
then it is immediate that $\calJ\models \ax$. Suppose there exists 
$\A\in\sig{\bigvee_{j=1}^m \fmm_j}$ such that $\A^{\calJ} \neq \Delta^{\arity(\A)}$.
Then $\A\in\S\cup\sig{\rel(\upchistar)}$.
If $\A\in\S$ then $\A(*,\dots,*)\in\arel^{\chistar}$.
Since $\A\in\sig{\bigvee_{j=1}^m\fmm_j}$ and 
$\fct*\in\prog^{\chistar}(\abox_0^{\mathsf{m}})$ for every $\fct\in\fml$, 
there is a proof $\rho_{\A,\ax}$ of $\A(*,\dots,*)$ in 
$\prog^{\chistar}\cup\abox_0^{\mathsf{m}}$ 
that mentions a rule in $\Xi^{\chistar}(\ax)$. Therefore $\ax\in\M^{\chistar}$.
If $\A\in\sig{\rel(\upchistar)}\setminus\S$ then some other $\fct'\in\arel^{\mathsf{m}}$
must be supported by a rule that has $\A$ in its signature.
More specifically, there must be a proof of $\fct'$ in $\prog^{\chistar}\cup\abox_0^{\mathsf{m}}$
that has a proof of $\A(*,\dots,*)$ as a subproof.
Replacing this subproof with $\rho_{\A,\ax}$ results in a proof of $\fct'$ in 
$\prog^{\chistar}\cup\abox_0^{\mathsf{m}}$ that mentions a rule in $\Xi^{\chistar}(\ax)$. 
Therefore in this case $\ax\in\M^{\chistar}$ too.
Now, since all rules in $\Xi^{\chistar}(\ax)$ have the same body as $\ax$, we have that 
$\sig{\fml}\subseteq \rel(\upchistar)\setminus\set{\bot}$ and therefore $\calI$ and $\calJ$ agree over $\sig{\fml}$. 
By assumption, $\calJ\models\fml\s$, so also $\calI\models\fml\s$;
furthermore $\calI\models\M^{\chistar}$ implies
$\calI\models\bigvee_{j=1}^m  \fmm_j\s$, which implies 
$\calJ\models\bigvee_{j=1}^m  \fmm_j\s$ because 
$\A^{\calI} \subseteq \A^{\calJ}$
for every predicate $\A\in\sig{\bigvee_{j=1}^m  \fmm_j}$. 
Since $\s$ is arbitrary, we conclude that $\calJ\models \ax$.
\end{proof}

\section{Depletingness and Preservation of Justifications}

\depletingness*
\begin{proof}
For $z\in\set{\mathsf{q},\mathsf{f},\mathsf{i},\mathsf{c}}$
the statement follows from Propositions \ref{prop:fq-insep-charact-rules}, 
\ref{prop:generalConsequencePreservation} and 
\ref{prop:cqPreservation} by the arguments already presented in the proofs for 
Theorems \ref{thm:datalogAtomicSubsumptionModules}, 
\ref{thm:disjDatalogModules} and \ref{thm:datalogCQModules}.

For $z=\mathsf{m}$, we will now show that $\T\setminus \M^{\chistar} \eqm \emptyset$.
Let $\calI$ be a model of $\emptyset$ and $\abox_{\calI}$ the ABox defined by $\calI$ over $\S$. 
Consider the datalog program $\prog = \bigcup_{\ax\in\T\setminus\M^{\chistar}}\Xi_{\chistar}(\ax),$
and the materialisation $\prog(\abox_{\calI})$ of $\prog$ w.r.t. $\abox_{\calI}$. 
We show that $\prog(\abox_{\calI})$ is a model of $\T\setminus\M$ that coincides with $\calI$ over $\S$. For this, it 
suffices to show the following two properties:
\begin{itemize}
\item All facts over $\S$ in $\prog(\abox_{\calI})$ must already be in $\abox_{\calI}$.\\
  Let $\fct\in\prog(\abox_{\calI})$ be a fact over $\S$. If
  $\fct\notin\abox_{\calI}$ then there must exist a proof $\rho$ of
  $\fct$ in $\prog\cup\abox_{\calI}$. Since $\abox_{\calI}$ only
  mentions predicates from $\S$ and $\prog\subseteq \prog^{\chistar}$,
  we can find a proof of $\fct*\in\arel^{\mathsf{m}}$ in
  $\prog^{\chistar}\cup\abox_0^{\mathsf{m}}$ that mentions the exact
  same rules as $\rho$.  Let $\rrule$ be a rule mentioned in $\rho$,
  there must exist $\srule\in\T\setminus\M^{\chistar}$ such that
  $\rrule\in\Xi^{\chistar}(\srule)$; however, because $\rrule$ is also
  mentioned in a proof of $\fct*$ in
  $\prog^{\chistar}\cup\abox_0^{\mathsf{m}}$, it must also be
  $\srule\in\M^{\chistar}$.  This is a contradiction, so
  $\fct\in\abox_{\calI}$.

\item $\bot\notin\prog(\abox_{\calI})$.\\
Suppose $\bot\in\prog(\abox_{\calI})$. Then there must be a proof $\rho$ of $\bot$ in 
$\prog\cup\abox_{\calI}$. Again, we can find a proof of $\bot$ in $\prog^{\chistar}\cup\abox_0^{\mathsf{m}}$ 
supported by the exact same rules as $\rho$. Following a similar argument as before, we 
conclude that $\bot\notin\prog(\abox_{\calI})$.\qedhere
\end{itemize}
\end{proof}

\presJustifications*
\begin{proof}
The claim follows from Propositions \ref{prop:fq-insep-charact-rules}, 
\ref{prop:generalConsequencePreservation} and \ref{prop:cqPreservation}
similarly to Theorems \ref{thm:datalogAtomicSubsumptionModules}, 
\ref{thm:disjDatalogModules} and \ref{thm:datalogCQModules}.
\end{proof}

\section{Module Containment}

\begin{definition}
\label{def:blaHomomorphism}
Let $\upchi = \langle \theta, \abox_0, \arel \rangle$ and $\upchi' = \langle \theta', \abox'_0, \arel' \rangle$ and let $\N$ and $\N'$ be the sets of constants mentioned in $\upchi$ and $\upchi'$, respectively.
A mapping $\mu:\N \to \N'$ is a \emph{homomorphism from $\upchi$ to $\upchi'$} if 
the following conditions hold:
\begin{inparaenum}[(i)]
\item $\theta' = \theta\mu$,
\item $\abox_0\mu \subseteq \abox'_0 $; and 
\item $\arel\mu \subseteq \arel'$. 
\end{inparaenum}
We write
$\upchi\hookrightarrow\upchi'$ if a homomorphism from $\upchi$ to $\upchi'$ exists.
\end{definition}

\begin{restatable}{theorem}{homorphismSubset} \label{thm:homomorphism-subset}
  If $\upchi,\upchi'$ are s.t. $\upchi\hookrightarrow\upchi'$, then
  $\M^{\chi} \subseteq \M^{\chi'}$.
\end{restatable}

\begin{proof}
Suppose $\upchi = \langle \theta, \abox_0, \arel \rangle$ and $\upchi' = \langle\theta', \abox_0', \arel' \rangle$ 
with $\N$  and $\N'$ the sets of constants mentioned in $\upchi$  and $\upchi'$, respectively.
Let $\mu$ be a homomorphism from $\upchi$ to $\upchi'$ and let $\ax\in\M^{\chi}$. 
Some $\fct\in\arel$ 
must be supported  in $\prog^{\chi}\cup\abox_0$ by a rule in $\Xi^{\chi}(\ax)$. 
Since, by assumption, $\fct\mu\in\arel'$, it suffices for us to show that
$\fct\mu$ is supported in $\prog^{\chi'}\cup\abox'_0$ by a rule from $\Xi^{\chi'}(\ax)$.


To this end we will show that for any rule $\ax\in\T$ and any fact $\fct$ such that there exists a proof 
$\rho = (T,\lambda)$ of $\fct$ in $\prog^{\chi}\cup\abox_0$ mentioning $\srule\in\Xi^{\chi}(\ax)$, 
there exists a proof of $\fct\mu$ in $\prog^{\chi'}\cup\abox'_0$ mentioning $\Xi^{\chi'}(\ax)$.
We will reason by induction on the depth $d$ of $\rho$---which must be at least $1$ since by assumption it uses $\srule$.

\begin{itemize}
\item [$d=1$] \ \\
$\ax$ must be of the form $\bigwedge_{i=1^n}\fcu'_i(\x) \to \exists \y. [\bigvee_{j=1}^m \fmm_j(\x,\y)]$
so $\srule$ must be 
	\begin{itemize}
	\item $\bigwedge_{i=1^n}\fcu'_i \to \fct'\theta$ with $\fct'\in\fmm_j$ for some $1\leq j \leq m$ if $m>0$ \\
          The $\lambda$-images of the leaves of $T$ must be
          $\fcu_1,\dots,\fcu_n \in \abox_0$ such that there exists a
          MGU $\s$ of $\fcu_i, \fcu'_i$ for every $1\leq i \leq n$
          satisfying $\fct = \fct'\theta\s$.
	By assumption, we have $\fcu_i\mu\in\abox'_0$ for every $1\leq i \leq n$, and also
	$\srule' = \bigwedge_{i=1^n}\fcu'_i \to \fct'\theta' \in \Xi^{\chi'}(\ax)$ where $\theta'=\theta\mu$.
	Consider $\s' = \s\mu$. 
	It is easy to see that $\mu\s\mu = \s\mu$ since the domain of $\s$ is disjoint with both the domain and 
	the range of $\mu$. Therefore $(\fcu_i\mu)\s\mu = \fcu_i\s\mu$ for every $1\leq i \leq n$.
	Furthermore, since $\s$ is a MGU of $\fcu_i, \fcu'_i$ for every $1\leq i \leq n$, we have that 
	$\s'$ is a MGU of $\fcu_i\mu, \fcu'_i$ for every $1\leq i \leq n$.
	Finally, since $\theta' = \theta\mu$, we have that 
	$\fct\mu = \fct'\theta\s\mu = \fct'\theta\mu\s\mu = \fct'\theta'\s'$ is a consequence of
	$\srule'$ and $\fcu_1\mu, \dots, \fcu_n\mu$, and hence we have a proof of 
	$\fct\mu$ in $\prog^{\chi'}\cup\abox_0'$ 
	supported by $\srule'\in\Xi^{\chi'}(\ax)$.
	
      \item  $\bigwedge_{i=1^n}\fcu'_i \to \bot$ if $m=0$ 	\\
	Then it must be $\fct = \bot$ and, as in the previous case,
        the $\lambda$-images of the leaves of $T$ must be
        $\fcu_1,\dots,\fcu_n \in \abox_0$ such that there exists a MGU
        $\s$ of $\fcu_i, \fcu'_i$ for every $1\leq i \leq n$.
	Also, $\fcu_i\mu\in\abox'_0$ for every $1\leq i \leq n$, $\srule \in \Xi^{\chi'}(\ax)$, and 
	$\s' = \s\mu$ is a MGU of $\fcu_i\mu, \fcu'_i$ for every $1\leq i \leq n$, so we have a proof of 
	$\fct\mu = \bot$ in $\prog^{\chi'}\cup\abox_0'$ supported by $\srule\in\Xi^{\chi'}(\ax)$.
	\end{itemize}

      \item [$d>1$]\ \\
	Let $v$ be the root of $T$, let $\fcu_1,\dots,\fcu_n \in
        \abox_0$ be the $\lambda$-images of the children of $v$ and
        let $\ax'\in\T$ be such that the $\lambda$-image of the
        edges connecting $v$ with its children in $T$ is a rule in
        $\Xi^{\chi}(\ax')$.  Either $\srule\in\Xi^{\chi}(r')$ or it is
        mentioned in some subproof of $\rho$.  Our induction
        hypothesis implies that for each each $\ax''\in\T$, if some
        $\fcu_i$ is supported in $\prog^{\chi}\cup\abox_0$ by a rule
        in $\Xi^{\chi}(\ax'')$, then also $\fcu_i\mu$ is supported in
        $\prog^{\chi'}\cup\abox'_0$ by some rule in
        $\Xi^{\chi'}(\ax'')$.  Therefore, in either case, following an
        argument similar to case $d=1$, we can construct a proof
        $\rho'$ of $\fct\mu$ in $\prog^{\chi'}\cup\abox'_0$ from a
        collection of proofs of $\fcu_1\mu, \dots, \fcu_n\mu$ in
        $\prog^{\chi'}\cup\abox'_0$ and a rule in $\Xi^{\chi'}(\ax')$
        in such a way that $\rho'$ mentions a rule in
        $\Xi^{\chi'}(\ax)$.\qedhere
\end{itemize}
\end{proof}

\hierarchy*
\begin{proof}
This follows immediately from Theorem \ref{thm:homomorphism-subset}.
%
%
%
%
%
\end{proof}

 \section{Optimality}

\begin{definition}
  Let $\T$ be a TBox. Let $\S\subseteq \sig{\T}$ and $\S'=\S\setminus\set{\bot}$. For each
  existentially quantified variable $y$ in $\T$, let $c_{y}$ be a
  fresh constant. Let $\theta=\mset{y\mapsto
    c_{y}}{y\text{ existentially quantified in $\T$}}$. Furthermore, for each pair
  $\langle\A,\B\rangle\in\S'\times\sig{\T}$, let $\c_{\A,\B}$ be a vector of
  fresh constants of size $\arity(\A)$. We define
  $\Psiimp_0(\T,\S)=\langle\theta,\abox^{\mathsf{i}_0}_0,\arel^{\mathsf{i}_0}\rangle$
  where
  \begin{itemize}
  \item $\abox_0^{\mathsf{i}_0}=\mset{\A(\c_{\A,\B})}{\A,\B\in\S',\A\ne\B,\arity(\A)=\arity(\B)}\cup\mset{\A(\c_{\A,\bot})}{A\in\S'}$
  \item $\arel^{\mathsf{i}_0}=\mset{\B(\c_{\A,\B})}{\A,\B\in\S',\A\ne\B,\arity(\A)=\arity(\B)}\cup\set{\bot}$
  \end{itemize}
  We define $\Psicls_0(\T,\S)=\langle\theta,\abox_0^{\mathsf{c}_0},\arel^{\mathsf{c}_0}\rangle$ where
  \begin{itemize}
  \item $\abox^{\mathsf{c}_0}_0=\mset{\A(\c_{\A,\B})}{\A\in\S',\B\in\sig{\T}\setminus\set{\bot},\A\ne\B,\arity(\A)=\arity(\B)}\cup\mset{\A(\c_{\A,\bot})}{A\in\S'}$
  \item $\arel^{\mathsf{c}_0}=\mset{\B(\c_{\A,\B})}{\A\in\S',\B\in\sig{\T}\setminus\set{\bot},\A\ne\B,\arity(\A)=\arity(\B)}\cup\set{\bot}$
  \end{itemize}
  For each predicate $\B\in\S$ and each 
  $\mathbf{v}\in\set{1,\dots,\arity(\B)}^{\arity(\B)}$, let
  $*_{\B,\mathbf{v}}^1,\dots,*_{\B,\mathbf{v}}^{\arity(\B)+1}$ be fresh constants. We define
  $\Psidd_0(\T,\S)=\langle\theta,\abox_0^{\mathsf{f}_0},\arel^{\mathsf{f}_0}\rangle$
  where
  \begin{itemize}
  \item $\abox_0^{\mathsf{f}_0}=\mset{\A(\mathbf{d})}{\A\in\S',\B\in\S,\mathbf{v}\in\set{1,\dots,\arity(\B)}^{\arity(\B)},\mathbf{d}\in\set{*_{\B,\mathbf{v}}^1,\dots,*_{\B,\mathbf{v}}^{\arity(\B)+1}}^{\arity(\A)},\A(\mathbf{d})\ne\B(*_{\B,\mathbf{v}}^{\mathbf{v}})}$
  \item $\arel^{\mathsf{f}_0}=\mset{\B(*_{\B,\mathbf{v}}^{\mathbf{v}})}{\B\in\S,\mathbf{v}\in\set{1,\dots,\arity(\B)}^{\arity(\B)}}$\qedhere
  \end{itemize}
\end{definition}

\begin{proposition} \label{prop:canonical-module-settings} %
  Let $z\in\set{\mathsf{i},\mathsf{c}}$. Then,
  for every $\T$ and $\S$, $\M^{\Psi^{z}(\T,\S)}=\M^{\Psi^{z}_0(\T,\S)}$.
\end{proposition}

\begin{proof}
It is easy to see that $\Psi^{z}_0(\T,\S) \hookrightarrow \Psi^{z}(\T,\S)$
for $z\in\set{\mathsf{i},\mathsf{c}}$. 
By Theorem \ref{thm:homomorphism-subset},
we thus have that $\M^{\Psi^{z}_0(\T,\S)} \subseteq \M^{\Psi^{z}(\T,\S)}$.

Before we continue, note that for each $z\in\set{\mathsf{i},\mathsf{c}}$ 
the datalog programs $\prog^{\Psi^{z}(\T,\S)}$ and $\prog^{\Psi^{z}_0(\T,\S)}$ coincide. For readability, 
we will denote this program with $\prog^{z}$.

Let $\rrule\in \M^{\Psi^{\mathsf{i}}(\T,\S)}$. 
Some fact $\fct\in\arel^{\chiimp}$ must be supported by a rule in 
$\Xi^{\Psi^{\mathsf{i}}(\T,\S)}(\rrule) = \Xi^{\Psi_0^{\mathsf{i}}(\T,\S)}(\rrule)$.
The fact $\fct$ must be either $\bot$ or $\B(\c_{\A})$ with $\A,\B\in\S'$.
It is easy to see how one can turn any proof of $\bot$ (resp. $\B(\c_{\A})$) in
$\prog^{\mathsf{i}}\cup\abox^{\mathsf{i}}_0$ into a proof of $\bot$ 
(resp. $\B(\c_{\A,\B})$) in $\prog^{\mathsf{i}}\cup\abox_0^{\mathsf{i}_0}$
that mentions the exact same rules.
By construction of $\Psi^{\mathsf{i}}$, both $\bot$ and $\B(\c_{\A,\B})$ are in $\arel^{\mathsf{i}_0}$,
so $\rrule\in \M^{\Psi^{\mathsf{i}}_0(\T,\S)}$. 
Therefore $\M^{\Psi^{\mathsf{i}}(\T,\S)} \subseteq \M^{\Psi^{\mathsf{i}}_0(\T,\S)}$.

The argument for $z=\mathsf{c}$ is analogous.
\end{proof}

\optimality*

\begin{proof}
  We show the claim for $\Psiimp$, the argument for
  $\Psicls$ is similar. Let $\S'=\S\setminus\set{\bot}$.

  For $\Psiimp$, suppose for contradiction there is a uniform,
  $\mathsf{i}$-admissible $\Psi$ and some $\T$ such that
  $\M^{\Psiimp(\T,\S)}\not\subseteq\M^{\Psi(\T,\S)}$. Then, by
  Proposition~\ref{prop:canonical-module-settings},
  $\M^{\Psiimp_0(\T,\S)}\not\subseteq\M^{\Psi(\T,\S)}$, and hence, by
  Theorem~\ref{thm:homomorphism-subset},
  $\Psiimp_0(\T,\S)\not\hookrightarrow\Psi(\T,\S)$. Let
  $\Psi(\T,\S)=\langle\theta',\abox'_0,\arel'\rangle$. Since
  $\Psiimp_0(\T,\S)\not\hookrightarrow\Psi(\T,\S)$, by construction of
  $\Psiimp_0$, there are two cases to consider:
  \begin{itemize}
  \item There are some $\A,\B\in\S'$ with $\arity(\A)=\arity(\B)$ such
    that for every vector $\c$ of size $\arity(\A)$ of constants
    mentioned in $\Psi$, $\A(\c)\notin\abox'_0$ or
    $\B(\c)\notin\arel'$. Let 
    \[\T'=\set{\A(\x)\to\B(\x)}\cup\mset{\to\exists y.Q_y(y)}{y\text{ existentially
    quantified in $\T$, $Q_y$ fresh for every $y$}}\] 
    Then $\Psi(\T',\S)=\Psi(\T,\S)$ (by uniformity), and
    hence $\M^{\Psi(\T',\S)}=\emptyset$. Since
    $\emptyset\not\models\A(\x)\to\B(\x)$, we have
    $\M^{\Psi(\T',\S)}\not\eqi\T'$.
  \item We have
    $\bot\notin\arel'$. Let
    \[\T'=\set{\A(\x)\to\bot}\cup\mset{\to\exists y.Q_y(y)}{y\text{
        existentially quantified in $\T$, $Q_y$ fresh for every $y$}}\]
    for some $\A\in\S'$. Then $\Psi(\T',\S)=\Psi(\T,\S)$ (by uniformity),
    and hence $\M^{\Psi(\T',\S)}=\emptyset$. Since
    $\emptyset\not\models\A(\x)\to\bot$, we have
    $\M^{\Psi(\T',\S)}\not\eqi\T'$.
  \end{itemize}
  In both cases, we obtain a contradiction to $\Psi$ being
  $\mathsf{i}$-admissible.
\end{proof}

\begin{proposition} \label{prop:psidd-zero-admissible}
  The family $\Psidd_0$ is $\mathsf{f}$-admissible.
\end{proposition}

\begin{proof}
  By Propositions~\ref{prop:fq-insep-charact-rules}
  and~\ref{prop:generalConsequencePreservation}, it suffices to show
  that, given a datalog rule $\rrule=\fml\to\fct$ and a substitution
  $\sigma$ mapping all variables in $\rrule$ to distinct constants, we
  can construct a mapping $\nu$ such that
  $\fml\sigma\nu\subseteq\abox_0^{\mathsf{f}_0}$ and
  $\fct\sigma\nu\in\arel^{\mathsf{f}_0}$. W.l.o.g. we can assume
  $\fct\notin\fml$ (otherwise $\rrule$ is a tautology and hence
  trivially entailed by $\M^{\Psidd_0(\T,\S)}$) and therefore
  $\fct\sigma\notin\fml\sigma$ by injectivity of $\sigma$.

  Let $\fct\sigma=\B(\c)$. We construct $\nu$ as follows. Let $\mu$ be
  an ordering of the constants in $\c$. We define $\nu$ such that
  $c\nu=*_{\B,\c\mu}^{c\mu}$ if $c\in\c$ and
  $c\nu=*_{\B,\c\mu}^{\arity(\B)+1}$ otherwise. Since
  $\c\mu\in\set{1,\dots,\arity(\B)}^{\arity(\B)}$ we have
  $\B(\c)\nu\in\arel^{\mathsf{f}_0}$. Moreover, every fact in
  $\fml\sigma$ is mapped by $\nu$ to a fact $\A(\mathbf{d})$ where
  $\A\in\S\setminus\set{\bot}$,
  $\mathbf{d}\in\set{*_{\B,\c}^1,\dots,*_{\B,\c}^{\arity(B)+1}}^{\arity(\A)}$,
  and $\A(\mathbf{d})\neq \B(\c)\nu$ since $\B(\c)\notin\fml\sigma$.
  Thus $\fml\sigma\nu\subseteq\arel^{\mathsf{f}_0}$, and the claim
  follows.
\end{proof}

\begin{proposition}
  The family $\Psidd$ is not $\mathsf{f}$-optimal.
\end{proposition}

\begin{proof}
  By Proposition~\ref{prop:psidd-zero-admissible}, it suffices to show
  that $\M^{\Psidd(\T,\S)}\not\subseteq\M^{\Psidd_0(\T,\S)}$ for some
  $\T$ and $\S$. Let $\T=\set{\A(x)\to\A(x)}$ and $\S=\set{\A}$. Then
  $\M^{\Psidd(\T,\S)}=\T\not\subseteq\emptyset=\M^{\Psidd_0(\T,\S)}$.
\end{proof}


\end{document}